%% file: iclr2027_conference.tex
\documentclass{article} 
\usepackage{iclr2027_conference,times}

\input{math_commands.tex}

\usepackage[utf8]{inputenc} 
\usepackage[T1]{fontenc}    
\usepackage{hyperref}       
\usepackage{url}            
\usepackage{amsfonts}       
\usepackage{nicefrac}       
\usepackage{microtype}      
\usepackage{xcolor}         
\usepackage{graphicx}
\usepackage{subcaption}
\usepackage{amsmath,amssymb,amsthm}
\newtheorem{definition}{Definition}
\newtheorem{proposition}{Proposition}
\usepackage{mathtools}
\usepackage{bm}               
\usepackage{bbm}              
\usepackage{booktabs}
\usepackage{multirow}
\usepackage{array}
\usepackage{colortbl}          
\usepackage{makecell}          
\usepackage{pifont}            
\usepackage{wrapfig}
\usepackage{placeins}         
\usepackage{float}            
\usepackage{needspace}        
\usepackage{algorithm}
\usepackage{algorithmic}        
\usepackage{natbib}
\usepackage{caption}
\usepackage{cleveref}         

\title{LILA: Data-Free Structured Pruning of LLMs via Latent Spectral Geometry}

\author{
Sankar Behera\\
Department of Computer Science and Engineering\\
Indian Institute of Technology Jammu\\
\texttt{sankar.behera@iitjammu.ac.in}
\And
Dhruv Singh\\
Department of Mathematics\\
Indian Institute of Technology Jammu\\
\texttt{2023uma0210@iitjammu.ac.in}
\And
Anshika Agnihotri\\
Department of Computer Science and Engineering\\
Indian Institute of Technology Jammu\\
\texttt{2025pcs0022@iitjammu.ac.in}
\AND
Raj Kumar Choudhary\\
Department of Information Technology\\
Engineering College Bikaner\\
\texttt{Choudhary.rajkumar@ecb.ac.in}
\And
Satyadev Ahlawat\\
Department of Electrical Engineering\\
Indian Institute of Technology Jammu\\
\texttt{satyadev.ahlawat@iitjammu.ac.in}
\And
Yamuna Prasad\\
Department of Computer Science and Engineering\\
Indian Institute of Technology Jammu\\
\texttt{yamuna.prasad@iitjammu.ac.in}
}

\iclrfinalcopy 
\begin{document}

\graphicspath{{figures/}}

\maketitle

\fancyhead[L]{}

\begin{abstract}
Structured pruning of large language models (LLMs) has universally required
calibration data, gradient computation, or large auxiliary policy networks at
pruning time.
LILA (\emph{Latent-Informed Layer Analysis}) demonstrates that this dependency
is avoidable: each FFN weight matrix is factorized with Non-negative Matrix
Factorization (NMF) to obtain a parts-based spectral representation, and neuron
importance is scored via the Kolmogorov--Smirnov (KS) distance between singular
value distributions of the full and neuron-ablated NMF matrix, a closed-form
criterion requiring no calibration corpus, forward pass, or gradient computation.
Two data-free matrix variants (Abs, Split) generate pruning masks from weight
spectral geometry alone; a calibration-guided variant (Act) is provided as an
internal ablation. 
Evaluation spans 7 model families (up to 70B). Without fine-tuning, data-free LILA
surpasses PruneNet (45M-parameter RL policy) by up to 2.12~pp in zero-shot accuracy
on LLaMA-2-7B and outperforms WikiText-2-calibrated SliceGPT by up to 6.1~pp across
all sparsity levels, confirming data-free performance is not an artifact of omitting
activation statistics.
After one LoRA RFT epoch (matched WikiText-2 protocol), LILA Act surpasses
SliceGPT+RFT by 6.39~pp on average using 98.4\% fewer calibration samples
(16 vs.\ 1,024) across LLaMA-2-7B and Phi-2.
NTK analysis confirms a $17{\times}$ reduction in functional distortion versus
random pruning, providing theoretical grounding for the spectral criterion.
Code: \url{https://anonymous.4open.science/r/LILA-33CB}
\end{abstract}

\input{sections/introduction}
\input{sections/related_work}
\input{sections/method}
\input{sections/experiments}
\input{sections/conclusion}


\subsection*{AI Use Statement}

In this work, we used generative AI tools to aid or polish writing.
We have not used generative AI tools for retrieval and discovery, 
research ideation or execution, drafting
sections of the paper, generating synthetic datasets, or proving
mathematical claims.
All AI-assisted text was reviewed, edited, and verified for accuracy by
the authors; any prose produced with AI assistance was rewritten to ensure
correctness and consistency with the experimental results.
LLM-generated code was verified and tested for correctness by 2 authors.
We take full responsibility for the final content of this work, including
all text, claims, and artifacts.

\subsection*{Reproducibility Statement}

All hyperparameters, model checkpoints, and evaluation benchmarks are
described in \cref{app:setup_details}.
The LILA pruning algorithm is given in full in \cref{sec:method}, with
its theoretical basis in \cref{sec:nmf_theory} and \cref{app:proof}.
Complete per-benchmark numerical results for every configuration are
provided in \cref{app:full_results}--\cref{app:rft_full}.
An anonymous implementation is available at
\url{https://anonymous.4open.science/r/LILA-33CB}.

\subsection*{Ethics Statement}

This work studies structured pruning of publicly available large language
models and does not involve human subjects, private data, or novel
datasets.
Pruning can reduce inference costs, which may broaden access to LLMs in
resource-constrained settings; however, compressed models may inherit or
amplify biases present in the original models.
No data collection, user studies, or potentially sensitive applications
are introduced.
We used only permissively licensed open-source models and standard
academic benchmarks throughout.

\bibliography{iclr2027_conference}
\bibliographystyle{iclr2027_conference}

\input{sections/appendix}

\end{document}

%% file: math_commands.tex
\usepackage{amsmath,amsfonts,bm}

\def\eqref#1{equation~\ref{#1}}

\def\1{\bm{1}}

\DeclareMathAlphabet{\mathsfit}{\encodingdefault}{\sfdefault}{m}{sl}
\SetMathAlphabet{\mathsfit}{bold}{\encodingdefault}{\sfdefault}{bx}{n}

\DeclareMathOperator*{\argmin}{arg\,min}

%% file: sections/introduction.tex

\section{Introduction}
\label{sec:intro}

Structured pruning of large language models (LLMs) rests on an assumption so
pervasive it is rarely questioned: importance must be inferred from data.
A calibration corpus is passed through the model, activation statistics are
accumulated, and neurons are ranked by their observed influence on the output.
Yet the weights of a trained LLM already encode a complete picture of what each
neuron does, compressed into the spectral geometry of the weight matrices
themselves.  If that geometry alone constitutes a sufficient importance signal,
calibration is not a necessity but an avoidable constraint.
This work demonstrates that it is sufficient: on LLaMA-2-7B, a closed-form
spectral criterion derived without any data surpasses the PruneNet RL pruning
policy and outperforms WikiText-2-calibrated SliceGPT by up to
6.1~pp in zero-shot accuracy.

Model compression is a practical necessity: the scale of modern
LLMs~\citep{brown2020gpt3,touvron2023llama2} is substantial, and a LLaMA-2-7B
model alone requires approximately 14~GB of GPU memory at half precision,
rendering single-GPU inference infeasible on commodity hardware.
Structured pruning is particularly attractive because removing entire neurons,
attention heads, or transformer layers yields dense weight tensors that are
immediately hardware-efficient, without requiring specialised sparse-tensor
support, unlike unstructured methods~\citep{frantar2023sparsegpt,sun2023wanda}
that produce irregular sparsity delivering speedups only on dedicated hardware.

Every competitive structured pruning method, however, imposes a data dependency
at pruning time, treating calibration as an architectural requirement rather
than a design choice.
LLM-Pruner~\citep{ma2023llmpruner} requires gradient computation over calibration
sequences.
SliceGPT~\citep{ashkboos2024slicegpt} consumes 1,024 WikiText-2 samples to
estimate and apply an \emph{irreversible PCA rotation} that permanently alters
the model architecture; at 70B scale, numerical stability further requires
double-precision (FP64) computation (SliceGPT Appendix~A.2).
Wanda~\citep{sun2023wanda} relies on calibration samples for input-activation
statistics.
PruneNet~\citep{prunenet2024} avoids an explicit corpus by training an RL policy
via REINFORCE~\citep{williams1992reinforce}, but its parameter count scales as
$O(d_{\text{intermediate}} \cdot d_{\text{hidden}})$ and its pairwise neuron
interaction matrix $W' \in \mathbb{R}^{n \times n}$ (PruneNet Eq.~3) grows
quadratically with FFN width, rendering the approach impractical at 70B scale.
Beyond resource cost, calibration data may be inaccessible in privacy-sensitive
deployments or mismatched to the target task, degrading pruning quality under
distribution shift.
The shared root cause is that all these methods treat data as essential, when the
open question is whether weight geometry alone is already sufficient.

The singular value spectrum of a weight matrix is a coordinate-free summary of
its principal action directions, encoding how much energy the layer concentrates
along each latent basis vector.
A neuron whose removal substantially perturbs this spectrum is spectrally
indispensable; a neuron whose ablation leaves the spectrum invariant is spectrally
redundant.
Because the spectrum is determined entirely by the weight values, it is an
\emph{intrinsic} property of the matrix, computable without any data and invariant
to calibration corpus choice.
Crucially, this makes neuron importance a purely geometric question: solvable by
closed-form matrix analysis rather than by data-dependent estimation.

\textbf{LILA} (\emph{Latent-Informed Layer Analysis}) exploits this observation
directly.
Non-negative Matrix Factorization (NMF)~\citep{lee1999nmf} is adopted as the
factorization of choice because its non-negativity constraint yields a parts-based
latent factor basis: neurons project onto non-overlapping latent directions with
no negative cancellations, making the per-neuron spectral contribution both
interpretable and well-defined.
LILA derives per-neuron importance scores from the Kolmogorov--Smirnov (KS)
distance between the singular value CDF of the full weight matrix and that of
each neuron-ablated variant, requiring no forward pass, no calibration data, and
no auxiliary network.

The main contributions of this work are as follows:
\begin{enumerate}

  \item \textbf{Unified NMF importance-scoring framework.}
        A single principled framework unifies data-free and calibration-guided
        neuron scoring on a continuous speed--accuracy continuum.
        Three matrix variants (data-free: Abs, Split; calibration-guided: Act) paired
        with three scoring criteria (Residual Energy, Reconstruction Sensitivity,
        KS-Distance Spectrum) span an ${\sim}24{\times}$ pruning-speed range
        (2.9--70.5 min on LLaMA-2-7B).
        Data-free Variants Abs and Split match or exceed calibrated Variant Act
        at three of five (model, sparsity) settings (by up to 1.85~pp) confirms
        spectral geometry alone as a sufficient importance signal at most.
        After one LoRA RFT epoch, LILA Act surpasses SliceGPT+RFT by 6.39~pp on average
        using 98.4\% fewer calibration samples (16 vs.\ 1,024) across LLaMA-2-7B and Phi-2.

  \item \textbf{Closed-form spectral criterion.}
        The Kolmogorov--Smirnov distance between singular value CDFs of the full
        and neuron-ablated weight matrix is introduced as a closed-form importance
        criterion.
        This mathematically rigorous approach surpasses the PruneNet RL policy
        by up to 2.12~pp in zero-shot accuracy on LLaMA-2-7B, achieving the best data-free result among the
        methods compared on that architecture, without any policy training.

  \item \textbf{Architecture-preserving empirical validation across 7 model families.}
        Evaluation spans LLaMA-2-7B/13B/70B, Phi-2, OPT-2.7B/6.7B, and Mistral-7B
        (2.7B--70B) at 20--30\% sparsity on five zero-shot benchmarks.
        Without fine-tuning, LILA-Spectrum outperforms WikiText-2-calibrated
        SliceGPT by up to 6.1~pp, including at 70B scale, using only standard
        FP32 (vs.\ SliceGPT's FP64 PCA on 1,024 samples). After recovery
        fine-tuning under a matched WikiText-2 protocol, LILA surpasses
        SliceGPT+RFT by 6.39~pp on average using 98.4\% fewer calibration
        samples.

  \item \textbf{NTK theoretical grounding.}
        The NTK trace ratio is introduced as a functional distortion metric;
        LILA achieves a 17$\times$ lower ratio than random pruning
        ($\rho{=}2.79$ vs.\ $47.38$, Phi-2 25\%) providing principled
        justification for the spectral importance criterion.

\end{enumerate}

The core pruning mask (Variants Abs and Split) requires no data: no calibration
corpus, no forward pass, no gradient computation.
Variant Act (16 samples) is a unified ablation showing the framework also subsumes
calibration-guided scoring; LoRA RFT is standard post-processing applied
identically by all baselines and leaves the data-free status of the pruning
criterion unchanged.
Collectively, these results establish spectral geometry as a sufficient and
principled pruning signal competitive with calibrated methods in zero-shot
evaluation and within margin of the strongest calibrated baselines after RFT.
NTK trace analysis grounds this in theory, linking spectral importance to
functional distortion, the quantity that governs downstream task performance.
The framework targets FFN intermediate neurons: FFN sub-layers hold $2{\times}$
more parameters than attention in MHA and $4.2$--$4.7{\times}$ more in GQA
models, with non-linear activations (ReLU, SwiGLU) exclusive to FFN layers
(\cref{sec:appendix_ffn_motivation}).

%% file: sections/related_work.tex

\section{Related Work}
\label{sec:related}

\textbf{Unstructured pruning.}
High sparsity ratios are achievable without modifying the model graph, but the
resulting irregular tensors deliver theoretical compression without practical
inference speedups on standard hardware.
SparseGPT~\citep{frantar2023sparsegpt} scales the classical OBS second-order
framework~\citep{lecun1989obd,hassibi1993obs} to billion-parameter models via
approximate inverse Hessians.
Wanda~\citep{sun2023wanda} scores weights by the product of magnitude and
input-activation norms.
Both require calibration data; LILA Variants Abs and Split operate fully
data-free (\cref{def:variants}).

\textbf{Structured pruning.}
Removing entire neurons, heads, or layers yields dense tensors compatible with
standard BLAS routines, but existing methods uniformly depend on data or
auxiliary computation at pruning time.
LLM-Pruner~\citep{ma2023llmpruner} prunes coupled structures via gradient-based
dependency analysis.
ShortGPT~\citep{men2024shortgpt} removes near-zero-contribution transformer
layers wholesale.
FLAP~\citep{flap2024} computes weight fluctuation statistics from 128
calibration samples to score FFN neurons; its data requirement and FFN focus
make it the closest architectural comparison to LILA.
SoBP~\citep{sobp2024} applies a block-level importance score derived from 128
WikiText-2 samples, achieving competitive accuracy but precluding use in
privacy-sensitive or data-inaccessible settings.
SliceGPT~\citep{ashkboos2024slicegpt} applies a PCA rotation and removes
trailing principal components, permanently restructuring the model graph; at
70B scale it additionally requires FP64 PCA for numerical stability
(SliceGPT, Appendix~A.2).
PruneNet~\citep{prunenet2024} reformulates pruning as policy learning, training
a policy network via REINFORCE~\citep{williams1992reinforce} whose parameter
count scales as $O(d_{\text{intermediate}} \cdot d_{\text{hidden}})$
(approximately 45M on LLaMA-2-7B, 71M on LLaMA-2-13B); although it avoids an
explicit calibration corpus, its pairwise neuron interaction matrix
$W' \in \mathbb{R}^{n \times n}$ (PruneNet, Eq.~3) grows quadratically with FFN
width, requiring approximately 3.3~GB per layer at LLaMA-2-70B scale and
rendering the approach impractical beyond 13B-class models.
LILA, by contrast, preserves the original architecture and replaces learned or
data-dependent criteria with a closed-form KS-distance spectral rule, adding no
auxiliary parameters and scaling linearly with the number of layers.


\textbf{Low-rank factorization.}
SVD-LLM~\citep{svdllm2024} and ASVD~\citep{asvd2023} approximate weight
matrices at reduced rank while retaining the same neuron count, requiring
specialised sequential matrix-vector product kernels for inference.
LILA instead uses NMF~\citep{lee1999nmf,lee2001nmf} as an
\emph{importance-scoring instrument} for discrete neuron deletion, yielding a
standard dense model with 1.21--1.64$\times$ native speedup without custom kernels;
prior NMF pruning~\citep{behera2026you} operated in a pre-LLM, at-initialization
setting distinct from LILA's.

\textbf{Neural Tangent Kernel analysis.}
The NTK~\citep{jacot2018ntk,lee2019ntk} characterises neural network training
in function space; \citet{wang2023ntksap} study the NTK spectrum in the context
of mask-based pruning at initialisation.
This work extends that analysis to post-training transformer FFN pruning via
the NTK trace ratio (\cref{sec:method:ntk}), providing a principled measure of
functional distortion that connects spectral importance scores to downstream
task performance.
Post-training quantization~\citep{frantar2022gptq} is orthogonal to spectral
pruning and composable in a pipeline.

\paragraph{Positioning.}
Unlike Wanda, FLAP, and SoBP, LILA-Spectrum derives importance from weight
spectral geometry alone, requiring no calibration data.
Unlike SliceGPT, it applies a structured neuron mask without modifying the
model graph or introducing auxiliary rotation matrices.
Unlike PruneNet, it adds no trainable policy parameters and scales to 70B
models without memory constraints.
Unlike SVD-LLM and ASVD, it produces a standard dense model rather than a
low-rank product, enabling immediate hardware throughput gains.
Full quantitative contrasts are given in \cref{sec:method} and
\cref{app:effective_sparsity}.
\Cref{tab:paradigm_comparison} (\cref{app:paradigm_comparison}) summarizes
these methodological trade-offs.

%% file: sections/method.tex

\section{Method}
\label{sec:method}

LILA formalises the pruning objective (\cref{sec:method:problem}), constructs a
non-negative weight matrix (\cref{sec:method:nmf}), scores neurons by their
spectral contribution to the NMF factorization and derives the pruning mask
(\cref{sec:method:scoring}); \cref{sec:method:ntk} provides NTK grounding.
\Cref{fig:architecture} illustrates the resulting six-stage pipeline.

\subsection{Problem Formulation}
\label{sec:method:problem}

\begin{figure*}[t]
\centering
\includegraphics[width=0.82\textwidth]{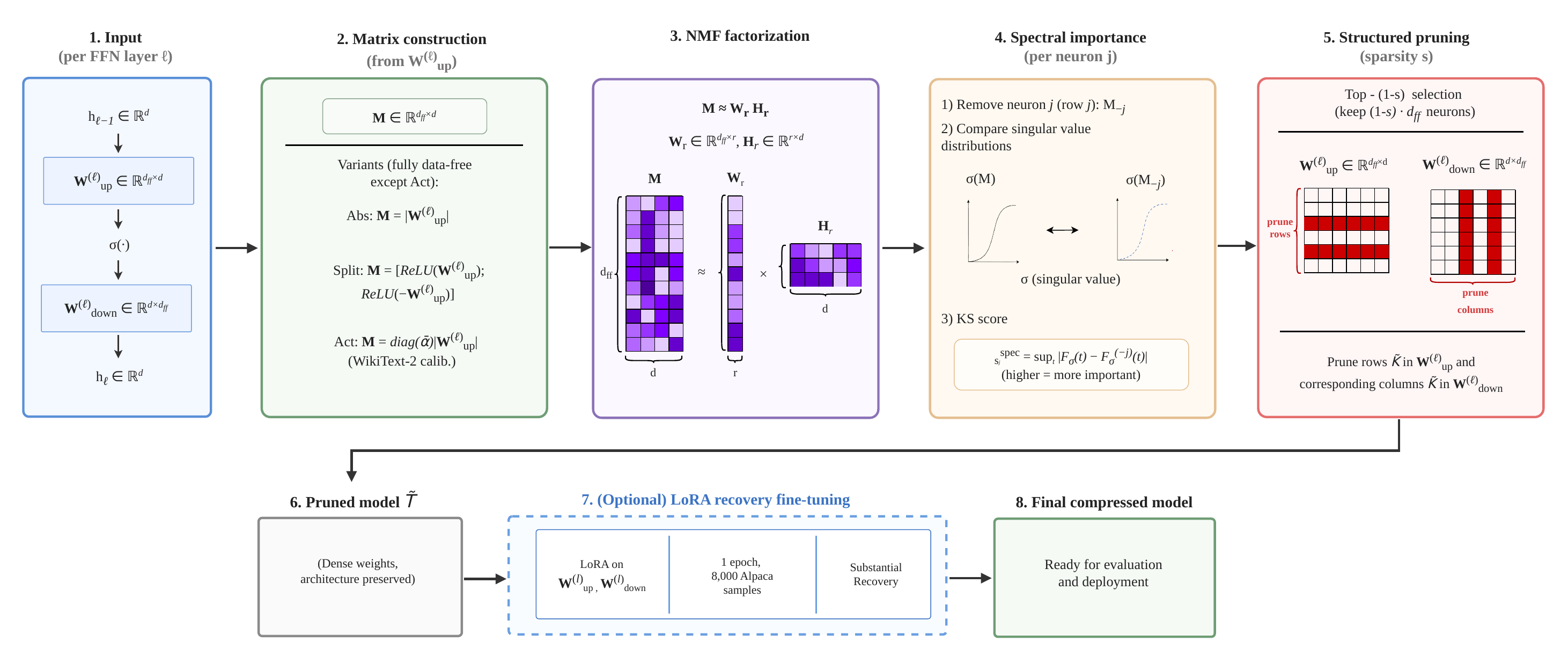}
\caption{%
\textbf{LILA framework.}
(1) For each FFN layer, a non-negative matrix $M$ is constructed
(2) from $W$ using one of three variants,
(3) factorized with rank-$r$ NMF, and
(4) each neuron $j$ is scored by the Kolmogorov--Smirnov distance between the singular value distributions of $M$ and the neuron-ablated matrix $M_j$.
(5) The lowest-scoring $(1 - s)$ fraction of neurons are pruned by removing rows in $W$ and corresponding columns in $W^{(e)}$.
(6) The resulting pruned LLM preserves the dense architecture and can be optionally fine-tuned with LoRA for recovery.}
\label{fig:architecture}
\end{figure*}

\paragraph{Transformer FFN layers.}
A transformer language model $\mathcal{T}$ with parameters $\Theta$ comprises $L$
stacked blocks, each containing a multi-head self-attention sub-layer followed by a
feed-forward network (FFN) sub-layer. The FFN in layer $\ell$ is parameterized by an
\emph{up-projection} $W_{\mathrm{up}}^{(\ell)} \in \mathbb{R}^{d_{\mathrm{ff}} \times d}$
and a \emph{down-projection} $W_{\mathrm{dn}}^{(\ell)} \in \mathbb{R}^{d \times d_{\mathrm{ff}}}$,
where $d$ denotes the model dimension and $d_{\mathrm{ff}} = 4d$ in standard
configurations. Given a hidden state $x \in \mathbb{R}^{1 \times d}$, the FFN
computes
\begin{equation}
  \mathrm{FFN}(x) \;=\; \sigma\!\left(x\,W_{\mathrm{up}}^{(\ell)T} + b_{\mathrm{up}}\right)
                          W_{\mathrm{dn}}^{(\ell)T} + b_{\mathrm{dn}},
  \label{eq:ffn}
\end{equation}
where $\sigma(\cdot)$ is a nonlinear activation (SiLU for LLaMA-2~\citep{touvron2023llama2};
GeLU for Phi-2~\citep{li2023phi2} and OPT~\citep{zhang2022opt}).\footnote{For
SwiGLU FFNs (LLaMA-2, Mistral), the gate projection $W_{\mathrm{gate}}$ is
pruned with the same row mask $\mathcal{K}_\ell$ as $W_{\mathrm{up}}$, with
corresponding column removal from $W_{\mathrm{dn}}$.}

\paragraph{Structured pruning formulation.}
Let $\mathcal{N}_\ell = \{1, \ldots, d_{\mathrm{ff}}\}$ denote the neuron index set
of layer $\ell$.  A \emph{pruning mask} $\mathcal{K}_\ell \subset \mathcal{N}_\ell$
identifies the set of neurons to be \emph{retained}:
\begin{equation}
  |\mathcal{K}_\ell| \;=\; \left\lfloor (1 - s)\,d_{\mathrm{ff}} \right\rfloor,
  \qquad \forall\,\ell \in \{1, \ldots, L\},
  \label{eq:mask_size}
\end{equation}
where $s \in (0, 1)$ is the global sparsity ratio. Pruning removes rows
$\mathcal{K}_\ell^c = \mathcal{N}_\ell \setminus \mathcal{K}_\ell$ from
$W_{\mathrm{up}}^{(\ell)}$ and the corresponding columns from $W_{\mathrm{dn}}^{(\ell)}$:
\begin{equation}
  \widetilde{W}_{\mathrm{up}}^{(\ell)} = W_{\mathrm{up}}^{(\ell)}\!\left[\mathcal{K}_\ell,\, :\right],
  \qquad
  \widetilde{W}_{\mathrm{dn}}^{(\ell)} = W_{\mathrm{dn}}^{(\ell)}\!\left[:,\, \mathcal{K}_\ell\right].
  \label{eq:pruned_weights}
\end{equation}
The resulting pruned model $\widetilde{\mathcal{T}}$ has a hidden FFN dimension
of $\lfloor(1-s)\,d_{\mathrm{ff}}\rfloor$ per layer and preserves the original
architecture's computation graph (no RMSNorm column deletion or graph rewiring).

The structured pruning problem is formulated as:
\begin{equation}
  \mathcal{K}^* \;=\; \argmin_{\substack{\mathcal{K}_\ell \subseteq \mathcal{N}_\ell \\
  |\mathcal{K}_\ell| = \lfloor(1-s)d_{\mathrm{ff}}\rfloor,\;\forall\ell}}
  \;\mathcal{L}\!\left(\widetilde{\mathcal{T}}_{\mathcal{K}},\, \mathcal{D}\right),
  \label{eq:pruning_objective}
\end{equation}
where $\mathcal{L}$ is the cross-entropy language modeling loss and $\mathcal{D}$
denotes the target data distribution.  Because access to $\mathcal{D}$ is assumed
unavailable (data-free setting: no calibration corpus, forward pass, or gradient
computation), the objective in \cref{eq:pruning_objective} is approximated via
an importance scoring function derived from the weight matrices alone.

\subsection{NMF Matrix Construction}
\label{sec:method:nmf}

Let $W \in \mathbb{R}^{d_{\mathrm{ff}} \times d}$ denote the up-projection weight
matrix of an arbitrary layer (layer superscripts are dropped for clarity). A
non-negative matrix $M \in \mathbb{R}^{m \times d}_+$ is constructed from $W$
according to one of three \emph{variants}, each encoding a different inductive
assumption about the weight geometry.

\begin{definition}[NMF Matrix Variants]
  \label{def:variants}
  \begin{align}
    M^{(A)} &= |W| \;\in\; \mathbb{R}^{d_{\mathrm{ff}} \times d}, \tag{Abs}
    \label{eq:varA} \\[4pt]
    M^{(B)} &= \bigl[\,\mathrm{ReLU}(W),\;\mathrm{ReLU}(-W)\bigr]
               \;\in\; \mathbb{R}^{d_{\mathrm{ff}} \times 2d}, \tag{Split}
    \label{eq:varB} \\[4pt]
    M^{(C)} &= \mathrm{diag}(\bar{a})\, |W| \;\in\; \mathbb{R}^{d_{\mathrm{ff}} \times d},
    \tag{Act}
    \label{eq:varC}
  \end{align}
  where $\bar{a}_j = \mathbb{E}_{x \sim \mathcal{D}_c}[|h_j(x)|]$ is the mean
  absolute activation of neuron $j$ estimated over a calibration corpus
  $\mathcal{D}_c$. Variants Abs and Split are \emph{data-free} (no calibration
  corpus, forward pass, or gradient required); Variant Act is a
  \emph{calibration-guided ablation}: all data-free results in this paper use
  Variants Abs or Split.
\end{definition}

\paragraph{Rationale.}
Variant Abs captures the raw magnitude geometry of $W$.
Variant Split decomposes $W$ into positive and negative components before
factorization, letting NMF model each signed manifold separately and yielding
a richer factorization basis.
Variant Act gates magnitude by empirical activation frequency (analogous
to~\citet{sun2023wanda}), applied as a matrix scaling prior to NMF; it confirms
that Abs/Split performance does not depend on omitting activation statistics.
NMF's non-negativity constraint is empirically decisive (\cref{sec:appendix_nmf_svd}):
substituting SVD collapses the Residual criterion's pre-RFT accuracy by up
to 15.3~pp on LLaMA-2-7B (54.16\% vs.\ 38.86\%), with the gap closing to
within 0.2~pp after LoRA recovery.

\paragraph{NMF factorization.}
Given the non-negative matrix $M \in \mathbb{R}^{m \times d}_+$, the
rank-$r$ NMF seeks factors $\mathbf{W}_r \in \mathbb{R}^{m \times r}_+$
and $\mathbf{H}_r \in \mathbb{R}^{r \times d}_+$ that minimize the
Frobenius reconstruction error:
\begin{equation}
  \min_{\mathbf{W}_r \geq 0,\; \mathbf{H}_r \geq 0}
  \left\| M - \mathbf{W}_r \mathbf{H}_r \right\|_F^2.
  \label{eq:nmf_objective}
\end{equation}
Problem~\eqref{eq:nmf_objective} is solved via the Multiplicative Update (MU)
rules of~\citet{lee1999nmf,lee2001nmf} with random initialization,
typically requiring only seconds per layer on a single GPU.

\subsection{Neuron Importance Scoring}
\label{sec:method:scoring}

Three complementary importance scores are derived from the NMF factorization
$M \approx \mathbf{W}_r \mathbf{H}_r$.  Let $M_j \in \mathbb{R}^{1 \times d}$
denote the $j$-th row of $M$, and let $(\mathbf{W}_r \mathbf{H}_r)_j$ denote the
$j$-th row of the NMF reconstruction matrix.

\paragraph{Residual Energy.}

The residual energy score measures how poorly the NMF approximation captures
neuron $j$'s contribution:
\begin{equation}
  s_j^{\mathrm{res}} \;=\; \left\| M_j - \left(\mathbf{W}_r \mathbf{H}_r\right)_j \right\|_2.
  \label{eq:residual_score}
\end{equation}
A high residual indicates that neuron $j$ carries information not representable
within the low-rank basis $\{\mathbf{W}_r, \mathbf{H}_r\}$, suggesting higher
importance.  Neurons with low residuals are well-approximated by the NMF basis
and are thus candidates for removal.

\paragraph{Reconstruction Sensitivity.}

The sensitivity score quantifies the magnitude of neuron $j$'s contribution
to the NMF reconstruction $M \approx \mathbf{W}_r \mathbf{H}_r$:
\begin{equation}
  s_j^{\mathrm{sens}} \;=\; \left\| \left(\mathbf{W}_r \mathbf{H}_r\right)_j \right\|_2.
  \label{eq:sensitivity_score}
\end{equation}
This is the $\ell_2$ norm of the $j$-th row of the reconstructed matrix,
measuring how much the NMF basis attributes to neuron $j$.
Neurons with large reconstruction norms are retained as more important.
Where Residual Energy captures what the NMF basis \emph{fails} to represent
about a neuron, Sensitivity captures what the basis \emph{attributes} to it;
together they probe importance from opposite directions.

\paragraph{Spectrum Score (KS Divergence).}

The \emph{spectrum score}, the primary contribution of this work, measures the
perturbation to the singular value distribution of $M$ upon neuron ablation.

\begin{definition}[Empirical Spectral CDF]
  \label{def:ecdf}
  Given the ordered singular values $\sigma_1 \geq \sigma_2 \geq \cdots \geq \sigma_r$
  of $M$ (estimated via randomized SVD~\citep{halko2011randomized}), the empirical
  cumulative distribution function (ECDF) of the spectrum is
  \begin{equation}
    F_\sigma(t) \;=\; \frac{1}{r} \sum_{k=1}^r \mathbf{1}[\sigma_k \leq t], \qquad t \in \mathbb{R}.
    \label{eq:ecdf}
  \end{equation}
  Analogously, $F_\sigma^{(-j)}(t)$ denotes the ECDF of the singular values of
  the ablated matrix $M_{-j}$.
\end{definition}

\begin{definition}[KS-Distance Spectrum Score]
  \label{def:ks_score}
  The spectrum importance score of neuron $j$ is the
  Kolmogorov–Smirnov (KS) statistic~\citep{kolmogorov1933sulla,smirnov1948table}
  between the full and ablated spectral ECDFs:
  \begin{equation}
    s_j^{\mathrm{spec}} \;=\; \sup_{t \in \mathbb{R}}\;
    \left| F_\sigma(t) - F_\sigma^{(-j)}(t) \right|.
    \label{eq:ks_score}
  \end{equation}
\end{definition}

\paragraph{Interpretation.}
A neuron with high $s_j^{\mathrm{spec}}$ is \emph{spectrally indispensable}: its
removal shifts the energy distribution of the weight manifold, contributing
directions not spanned by the remaining neurons; near-zero scores indicate
spectral redundancy and safe removal.

\paragraph{Computational efficiency.}
Because the neuron-ablation perturbation is rank-1, top-$r$ singular values of
$M_{-j}$ are computed via randomized SVD~\citep{halko2011randomized} per
neuron, yielding near-exact KS distances at one-time, gradient-free cost
(closed-form perturbation identity and a possible $O(r^2)$-per-neuron
acceleration in \cref{app:setup_details}).
For Variants Abs and Split, the full pipeline, including matrix construction,
requires no calibration data.

\paragraph{Pruning mask construction.}

Given importance scores $\{s_j\}_{j=1}^{d_{\mathrm{ff}}}$ (from any of the three
methods above), the pruning mask is constructed by retaining the top-$(1-s)$ fraction:
\begin{equation}
  \mathcal{K} \;=\;
  \left\{ j \;\middle|\;
    \mathrm{rank}\!\left(-s_j\right) \leq \left\lfloor (1-s)\, d_{\mathrm{ff}} \right\rfloor
  \right\},
  \label{eq:mask_construction}
\end{equation}
where $\mathrm{rank}(-s_j)$ denotes the rank of $-s_j$ in ascending order
(i.e., neurons are sorted by descending importance).
The above formulation applies a uniform sparsity ratio $s$ across all layers,
matching the comparison protocol of SliceGPT and PruneNet.
Because KS scores natively provide layer-wise sensitivity estimates, the same
scores support adaptive sparsity allocation without modification; this extension
is evaluated in \cref{app:adaptive_sparsity}.
The complete pruning procedure follows directly from
\cref{eq:varA,eq:varB,eq:varC,eq:nmf_objective,eq:ks_score,eq:mask_construction}
above; full implementation and hyperparameter details appear in
\cref{app:setup_details}.

\subsection{Theoretical Analysis: NTK Preservation}
\label{sec:method:ntk}

A functional justification for spectrum-based scoring is provided via the
Neural Tangent Kernel (NTK) framework~\citep{jacot2018ntk}, which characterises
functional distortion through the NTK trace ratio $\rho$.
As derived in \cref{app:ntk}, preserving the weight matrix spectral geometry
directly minimises this functional distortion.
Empirically, LILA-Residual (Act) achieves $\rho = 2.79$ ($17{\times}$ reduction
vs.\ random pruning at $\rho = 47.38$) on Phi-2 at 25\% sparsity;
LILA-Spectrum (Split) attains $\rho = 3.14$ ($15{\times}$ reduction), confirming that
both scoring criteria substantially preserve the dense model's linearised
learning dynamics.
The $17{\times}$ NTK reduction certifies that LILA masks do not arbitrarily
corrupt learning geometry (empirical results in \cref{app:ntk}; formal
distortion-bound proof sketch in \cref{app:proof}).

%% file: sections/experiments.tex

\section{Experiments}
\label{sec:experiments}
\input{tables/tab_main_results}

\subsection{Experimental Setup}
\label{sec:experiments:setup}

LILA is evaluated across \textbf{seven model families} spanning 2.7B to 70B
parameters: LLaMA-2 (7B, 13B, 70B)~\citep{touvron2023llama2},
Phi-2~\citep{li2023phi2}, OPT (2.7B, 6.7B)~\citep{zhang2022opt}, and
Mistral-7B~\citep{jiang2023mistral}. Primary ablation experiments use
LLaMA-2-7B and Phi-2; generalization results across the remaining five model
families are presented in \cref{sec:experiments:generalization}.
Sparsity ratios $s \in \{0.20, 0.25, 0.30\}$ are applied uniformly
across all FFN layers.  Nine LILA configurations are evaluated: three scoring
methods (Spectrum, Sensitivity, Residual) combined with three NMF variants
(Abs, Split, Act; see \cref{def:variants}), all using NMF rank $r = 32$.
\input{tables/tab_generalization}
Baselines include PruneNet~\citep{prunenet2024} (45M-parameter RL policy)
and SliceGPT~\citep{ashkboos2024slicegpt} (PCA rotation, architecture permanently
modified; results sourced from the original paper). Wanda~\citep{sun2023wanda}
(128 WikiText-2 calibration samples) is included as an additional reference
in \cref{tab:app_llama_full}.
Evaluation metrics, recovery fine-tuning protocol, and hardware details are
provided in \cref{app:setup_details}.

\subsection{Results}
\label{sec:experiments:main}

\Cref{tab:main_results} presents a comprehensive multi-sparsity evaluation of all LILA configurations against external baselines on LLaMA-2-7B and Phi-2, detailing their calibration requirements, architectural footprints, and effective sparsity profiles (see \cref{app:effective_sparsity} for full derivation).

\paragraph{LILA-Spectrum outperforms noRFT baselines at most sparsities.}
Without any fine-tuning, data-free LILA-Spectrum surpasses PruneNet's
45M-parameter RL policy on LLaMA-2-7B at 20\% (Split, $+1.03$~pp) and 30\%
sparsity (Split, $+2.12$~pp; Abs, $+1.12$~pp), but trails PruneNet at 25\%
sparsity (Abs, $-0.84$~pp; Split, $-1.43$~pp). With the calibration-guided
Act variant, LILA surpasses PruneNet at all three sparsities ($+0.35$,
$+0.45$, and $+1.60$~pp at 20\%, 25\%, and 30\% respectively), all via a
closed-form rule requiring no policy training or gradient computation. Margins over
WikiText-2-calibrated SliceGPT are larger at most sparsity levels tested
($+2.3$ to $+6.1$~pp across LLaMA-2-7B and Phi-2;
\cref{tab:app_slicegpt_conditions}), with one exception: at Phi-2 30\%
sparsity, the data-free variant (Abs, $51.96\%$) trails WikiText-2-calibrated
SliceGPT ($51.99\%$) by $0.03$~pp. Against the
stronger Alpaca-calibrated SliceGPT, the data-free variant trails by at most
$11.51$~pp using its best per-sparsity configuration; full per-condition
results appear in
\cref{tab:app_slicegpt_conditions} (\cref{app:slicegpt_conditions}).

\paragraph{Under matched WikiText-2 protocol, LILA surpasses calibrated SliceGPT+RFT by 6.39~pp.}
\Cref{tab:rft_comparison} presents the post-RFT comparison with calibration
\emph{and} RFT held on WikiText-2 for both methods, isolating calibration
sample count (16/0 vs.\ 1{,}024) as the only varying factor and taking the
best-performing LILA configuration per setting (per-configuration sweep in
\cref{tab:app_rft_dataset}, \cref{app:slicegpt_conditions}).

\input{tables/tab_rft_comparison}

Under SliceGPT's own strongest reported protocol (1{,}024 Alpaca calibration + Alpaca
RFT), LILA using 98.4\% fewer or zero calibration samples remains
within a $-$0.48~pp mean gap instead (full 9-config comparison in
\cref{tab:rft_full}, \cref{app:rft_full}).

\paragraph{Calibration robustness.}
No single variant dominates across all settings: adding a calibration corpus (Act, WikiText-2) leads at some sparsity levels and trails at others relative to data-free Abs/Split, with gaps typically under 1~pp and reaching $1.85$~pp in the largest case (full analysis in \cref{app:ablation}). Critically, Act variants of the Sensitivity and Residual criteria, previously observed to lag sharply without fine-tuning, now perform comparably to their pre-RFT Spectrum counterparts across all sparsity levels, indicating the earlier instability was not an intrinsic property of calibration-guided scoring.

\paragraph{Architecture-conditional scoring.}
The optimal scoring criterion exhibits a systematic dependence on FFN activation architecture. Gated FFNs (LLaMA-2, OPT; SwiGLU/SiLU) isolate neuron contributions cleanly, making data-free Spectrum optimal at most sparsity levels. Non-gated FFNs (Phi-2; GeLU) entangle neuron contributions across positions, making activation-aware Sensitivity Act optimal post-RFT. This pattern is consistent with findings in FLAP~\citep{flap2024} and SoBP~\citep{sobp2024} and will be analyzed formally in future work.

\input{tables/tab_timing_tiny}

\paragraph{Generalization across large scale architectures.}
\label{sec:experiments:generalization}

To validate structural robustness, evaluation is extended to three additional architectures: OPT-6.7B, Mistral-7B, and LLaMA-2-70B (\cref{tab:generalization}; full multi-variant breakdown in \cref{tab:app_generalization_full}, \cref{app:generalization_full}). The data-free variant transfers out-of-the-box, outperforming WikiText-2-calibrated SliceGPT on all three architectures ($+0.77$ to $+6.15$~pp). Adding 16-sample calibration (Act) widens this margin further on every architecture ($+2.45$ to $+7.83$~pp): on Mistral-7B, where SliceGPT reports no results, LILA reaches $63.31\%$, exceeding SliceGPT's stronger Alpaca-calibrated LLaMA-2-7B result ($60.91\%$) by $+2.40$~pp; at 70B scale, the calibration-aided Residual variant reaches $72.76\%$ ($+3.01$~pp over SliceGPT FP64) using only 16 samples versus SliceGPT's 1,024 and standard FP32 versus SliceGPT's FP64 PCA requirement (full results in \cref{tab:app_llama70b}). LILA is validated across all 7 model families spanning 2.7B to 70B parameters (full LLaMA-2-13B/OPT-2.7B sweeps in \cref{app:scaling}).

\paragraph{Recovery fine-tuning across all configurations.}
\label{sec:experiments:rft}
In a protocol-matched comparison using the same WikiText-2 LoRA as PruneNet,
the best-performing LILA configuration leads by 1.10~pp (61.15\% vs.\ PruneNet+RFT: 60.05\%;
full breakdown in \cref{app:prunenet_comparison}).
The recovery dataset quality plays a decisive role: fine-tuning on instruction-following
data (Alpaca) yields up to 2.43~pp over unstructured text (WikiText-2),
confirming dataset sensitivities mirrored by SliceGPT (see \cref{tab:app_rft_dataset}).
The full post-RFT evaluation across all nine LILA configurations against
SliceGPT+RFT and PruneNet+RFT is provided in \cref{app:rft_full}, and the
complete per-benchmark breakdown for every configuration in \cref{app:full_results}.

\needspace{1.8in}
\begin{wraptable}{r}{0.50\textwidth}
\vspace{-1.0em}
\centering
\caption{Inference throughput (tok/s) of dense vs.\ pruned models (25\% sparsity). Yields out-of-the-box speedups without custom sparse kernels.}
\label{tab:app_throughput}
\small
\setlength{\tabcolsep}{4pt}
\renewcommand{\arraystretch}{1.13}
\resizebox{\linewidth}{!}{%
\begin{tabular}{@{}l r r r@{}}
\toprule
\textbf{Model} & \textbf{Dense (tok/s)} & \textbf{Pruned 25\%} & \textbf{Speedup} \\
\midrule
LLaMA-2-13B & 16.42 & 26.90 & 1.64$\times$ \\
LLaMA-2-7B  & 25.60 & 33.96 & 1.33$\times$ \\
Phi-2        & 35.06 & 42.34 & 1.21$\times$ \\
\bottomrule
\end{tabular}}
\vspace{-0.5em}
\end{wraptable}

\paragraph{Hardware throughput.}
\label{sec:experiments:efficiency}

Because LILA produces a standard dense transformer at reduced $d_{\mathrm{ff}}$, structured pruning yields direct hardware speedups without custom kernels. This increases inference throughput by up to $1.64\times$ at 25\% sparsity (\cref{tab:app_throughput}; full benchmarking protocol in \cref{app:throughput}).
While LILA-Spectrum requires 54.3--70.5 minutes for pruning (Phi-2 to LLaMA-2-7B; \cref{tab:timing_tiny}) due to per-neuron randomized SVD, this is comparable to SliceGPT's 44--49 minutes despite SliceGPT running on a substantially faster H100-80GB GPU~\citep{ashkboos2024slicegpt}. Meanwhile, LILA-Sensitivity combined with LoRA recovery attains over 98\% of the relative post-RFT performance at ${\sim}19\times$ lower pruning cost (3.3 minutes on LLaMA-2-7B). Full wall-clock time profiling and tradeoffs are detailed in \cref{app:efficiency}.

\needspace{2.6in}
\begin{wraptable}{l}{0.52\textwidth}
\vspace{-1.2em}
\centering
\caption{Uniform vs.\ Adaptive Sparsity. Adaptive allocation improves perplexity at 25\%. At 30\%, LLaMA-2 avoids catastrophic collapse by safely clamping structural limits (actual sparsity reaches 26.61\%).}
\label{tab:app_adaptive}
\small
\setlength{\tabcolsep}{4pt}
\renewcommand{\arraystretch}{1.13}
\resizebox{\linewidth}{!}{%
\begin{tabular}{@{}llccc@{}}
\toprule
\textbf{Model} & \textbf{Strategy} & \textbf{Act.\ $s$} & \textbf{Acc} & \textbf{PPL} \\
\midrule
\multirow{2}{*}{\textbf{Phi-2}}
& Uniform (Abs) & 25.00\% & \textbf{57.45} & 43.30 \\
& \textbf{Adaptive} & \textbf{24.36\%} & 55.54 & \textbf{30.78} \\
\midrule
\multirow{4}{*}{\textbf{LLaMA-2}}
& Uniform (Abs) & 25.00\% & 57.79 & 10.92 \\
& \textbf{Adaptive} & \textbf{22.84\%} & \textbf{61.40} & \textbf{10.23} \\
\cmidrule{2-5}
& Uniform (Abs) & 30.00\% & 56.57 & 12.49 \\
& \textbf{Adaptive} & \textbf{26.61\%} & \textbf{58.45} & \textbf{11.50} \\
\bottomrule
\end{tabular}}
\vspace{-1em}
\end{wraptable}
\paragraph{Additional analyses.}
\label{sec:experiments:additional}

Top-$k$ removal (LLaMA-2-7B) shows accuracy degrades sharply when the
highest-scored neurons are removed first, confirming the KS criterion identifies
functionally critical neurons (\cref{fig:app_topk}, \cref{app:layerwise}).
NMF rank sensitivity ($r \in \{8,16,32,64\}$) for the Sensitivity and Residual
criteria at 20\%/30\% sparsity yields accuracy spreads of at most 1.3~pp in
7 of 8 model-sparsity rows (one exception: LLaMA-2-7B Residual at 20\%, 2.5~pp);
$r{=}32$ is a reliable default, avoiding the sharp perplexity degradation at
$r{=}64$ in that row (\cref{app:rank_sweep}).

An adaptive sparsity extension uses the KS-score to allocate layer-wise budgets
dynamically. At 25\% compression on LLaMA-2, it improves both accuracy
(57.79\%$\to$61.40\%) and perplexity (10.92$\to$10.23) over uniform allocation;
on Phi-2, perplexity improves substantially (43.30$\to$30.78) at a moderate
accuracy cost. At 30\% compression, the allocator clamps effective sparsity to
26.61\% to avoid a single-layer bottleneck, still outperforming uniform pruning
on both metrics (Acc 58.45\% vs.\ 56.57\%; PPL 11.50 vs.\ 12.49)
(\cref{tab:app_adaptive}).



%% file: tables/tab_main_results.tex
\begin{table}[ht]
\centering
\caption{%
  \textbf{Comparative Evaluation of Structured Pruning Methods.}
  Zero-shot accuracy (\%\,$\uparrow$): mean over 5 benchmarks (PIQA, HellaSwag, ARC-E/C, WinoGrande).
  \textbf{Bold}: best calibration-free result per column.
  \underline{Underline}: overall best per column.
  \textsuperscript{\dag} WikiText-2 calibrated; architecture-modifying (PCA rotation), no fine-tuning.
  \textsuperscript{\S} Requires 45M-parameter RL policy.
  \textsuperscript{\textdaggerdbl} Effective sparsity at 25\% nominal (see \cref{app:effective_sparsity}).
  Post-RFT results for all nine LILA configurations appear in \cref{tab:rft_full}.
}
\label{tab:main_results}
\small
\setlength{\tabcolsep}{5pt}
\renewcommand{\arraystretch}{1.13}
\resizebox{\textwidth}{!}{
\begin{tabular}{@{}
  l                   
  c c                 
  c c c c             
  c c                 
  @{}}
\toprule
& & &
  \multicolumn{4}{c}{\textbf{LLaMA-2-7B}} &
  \multicolumn{2}{c}{\textbf{Phi-2}} \\
\cmidrule(lr){4-7}\cmidrule(lr){8-9}
\textbf{Method} &
  \textbf{Calib.} & \textbf{Arch.\,OK} &
  \textbf{Eff.Sp.\textsuperscript{\textdaggerdbl}} & \textbf{20\%} & \textbf{25\%} & \textbf{30\%} &
  \textbf{25\%}  & \textbf{30\%} \\
& & & (at 25\% nom.) & \multicolumn{5}{c}{\small Zero-Shot Accuracy (\%\,$\uparrow$)} \\
\midrule
Dense (unpruned) &
  --- & \ding{51} &
  0.0\% & 69.00 & 69.00 & 69.00 & 72.24 & 72.24 \\
\midrule
\multicolumn{9}{@{}l}{\small\textit{Prior methods (no fine-tuning)}} \\[1pt]
SliceGPT\textsuperscript{\dag}~\citep{ashkboos2024slicegpt} &
  1\,024 samp & \ding{55} &
  15.3\% & 58.18 & 55.48 & 51.50 &
  54.46 & 51.99 \\
PruneNet\textsuperscript{\S}~\citep{prunenet2024} &
  RL policy & \ding{51} &
  16.0\% & 61.67 & 58.63 & 55.45 &
  64.10 & 61.05 \\
\midrule
\multicolumn{9}{@{}l}{\small\textit{LILA (ours) -- zero-shot, no fine-tuning}} \\[1pt]
LILA-Spectrum, Abs &
  \textbf{None} & \ding{51} &
  16.0\% & 61.48 & \textbf{57.79} & 56.57 &
  \textbf{\underline{57.45}} & \textbf{51.96} \\
LILA-Spectrum, Split &
  \textbf{None} & \ding{51} &
  16.0\% & \textbf{\underline{62.70}} & 57.20 & \textbf{\underline{57.57}} &
  55.09 & 48.14 \\
LILA-Spectrum, Act &
  16 samp & \ding{51} &
  16.0\% & 62.02 & \underline{59.08} & 57.05 &
  55.60 & \underline{52.55} \\
LILA-Sensitivity, Act &
  16 samp & \ding{51} &
  16.0\% & 57.72 & 54.56 & 50.73 &
  54.73 & 51.02 \\
LILA-Residual, Act &
  16 samp & \ding{51} &
  16.0\% & 57.94 & 54.16 & 49.87 &
  55.58 & 51.46 \\
\bottomrule
\end{tabular}
}
\end{table}

%% file: tables/tab_generalization.tex
\needspace{2.6in}
\begin{wraptable}{r}{0.62\textwidth}
\vspace{-1.5em}
\centering
\caption{%
  Zero-shot accuracy (\%\,$\uparrow$) at 25\% sparsity (no fine-tuning).
  Both the data-free variant and calibrated (Act) variant
  outperform SliceGPT on all three architectures; full multi-variant
  breakdown in \cref{tab:app_generalization_full} (\cref{app:generalization_full}).
  \textsuperscript{*}LLaMA-2-7B value shown; SliceGPT does not evaluate Mistral-7B.
  \textsuperscript{\dag}Requires FP64 PCA at 70B scale (SliceGPT App.~A.2).
  \textsuperscript{\ddag}Best LILA: Spectrum, Sensitivity, Residual
  for OPT-6.7B/Mistral-7B/LLaMA-2-70B respectively.
}
\label{tab:generalization}
\small
\setlength{\tabcolsep}{4pt}
\renewcommand{\arraystretch}{1.13}
\resizebox{\linewidth}{!}{%
\begin{tabular}{@{}lccc@{}}
\toprule
\textbf{Method} & \textbf{OPT-6.7B} & \textbf{Mistral-7B} & \textbf{LLaMA-2-70B} \\
\midrule
SliceGPT (WikiText-2)\textsuperscript{\dag} & 54.28 & 55.48\textsuperscript{*} & 69.75 \\
\midrule
LILA-Spectrum, Split (\textbf{data-free}) & 56.28 & 61.63 & 70.52 \\
LILA best (16 samp)\textsuperscript{\ddag} & \textbf{56.73} & \textbf{63.31} & \textbf{72.76} \\
\midrule
\textit{$\Delta$ vs.\ SliceGPT (None)} & \textit{+2.00} & \textit{+6.15} & \textit{+0.77} \\
\textit{$\Delta$ vs.\ SliceGPT (16 samp)} & \textit{+2.45} & \textit{+7.83} & \textit{+3.01} \\
\bottomrule
\end{tabular}}
\vspace{-1.0em}
\end{wraptable}

%% file: tables/tab_rft_comparison.tex

\needspace{2.6in}
\begin{wraptable}{r}{0.50\textwidth}
\vspace{-1.0em}
\centering
\caption{%
  Post-RFT head-to-head under a matched WikiText-2 protocol: SliceGPT
  and LILA both calibrated \emph{and} recovery-fine-tuned on WikiText-2
  (\citealt{ashkboos2024slicegpt} Table~9), isolating calibration sample count
  (16/0 vs.\ 1{,}024) as the only varying factor.
  $\Delta$ = LILA\,+\,RFT $-$ SliceGPT\,+\,RFT; best LILA configuration per
  setting. Full comparison under SliceGPT's Alpaca+RFT protocol in
  \cref{tab:rft_full}.
}
\label{tab:rft_comparison}
\small
\setlength{\tabcolsep}{4pt}
\renewcommand{\arraystretch}{1.13}
\resizebox{\linewidth}{!}{%
\begin{tabular}{@{}llccc@{}}
\toprule
\textbf{Model} & \textbf{SP}
  & \textbf{LILA+RFT}
  & \textbf{SliceGPT+RFT}
  & $\boldsymbol{\Delta}$ \\
\midrule
LLaMA-2-7B & 20\% & 63.09 & 57.27 & $+$5.82\,pp \\
LLaMA-2-7B & 25\% & 61.15 & 56.20 & $+$4.95\,pp \\
LLaMA-2-7B & 30\% & 58.99 & 54.23 & $+$4.76\,pp \\
\midrule
Phi-2      & 25\% & 62.58 & 55.17 & $+$7.41\,pp \\
Phi-2      & 30\% & 60.72 & 51.70 & $+$9.02\,pp \\
\midrule
\multicolumn{3}{@{}l}{\textit{Mean gap (all 5 settings)}} & & $+$6.39\,pp \\
\bottomrule
\end{tabular}}
\vspace{-0.5em}
\end{wraptable}

%% file: tables/tab_timing_tiny.tex
\needspace{2.4in}
\begin{wraptable}{r}{0.52\textwidth}
\vspace{-1.0em}
\centering
\caption{%
  Average pruning wall-clock time (min; scoring + masking only) on NVIDIA A100-40GB, 
  averaged across all LILA algorithmic variants.
  $\dagger$: SliceGPT measured on \textbf{H100-80GB}~\citep{ashkboos2024slicegpt}
  (faster than A100; values are thus a Average accross multiple runs for A100).
  $\star$: PruneNet~\citep{prunenet2024} excludes RL policy pre-training.
  Full protocol in \cref{app:efficiency}.
}
\label{tab:timing_tiny}
\small
\setlength{\tabcolsep}{4pt}
\renewcommand{\arraystretch}{1.13}
\resizebox{\linewidth}{!}{
\begin{tabular}{@{}l c c c@{}}
\toprule
\textbf{Method} & \textbf{Calib.} & \textbf{LLaMA-2-7B} & \textbf{Phi-2} \\
\midrule
\textbf{LILA-Spectrum}    & \ding{55} & 70.5 & 54.3 \\
\textbf{LILA-Sensitivity} & \ding{55} & 3.3  & 1.3  \\
\textbf{LILA-Residual}    & \ding{55} & 2.9  & 1.4  \\
\midrule
PruneNet$^\star$~\citep{prunenet2024} & RL policy & 15.3 & 5.7 \\
SliceGPT$^\dagger$~\citep{ashkboos2024slicegpt} & 1{,}024 samp & 44~(H100) & 49~(V100) \\
\bottomrule
\end{tabular}}
\vspace{-0.5em}
\end{wraptable}

%% file: sections/conclusion.tex

\section{Conclusion}
\label{sec:conclusion}


LILA establishes spectral weight geometry as a sufficient and principled pruning
signal for LLMs: without any calibration data, LILA-Spectrum surpasses
WikiText-2-calibrated SliceGPT by up to 6.1~pp and PruneNet's RL policy by
up to 2.12~pp in zero-shot accuracy on LLaMA-2-7B across seven model families.
After one LoRA RFT epoch (matched WikiText-2 protocol), LILA Act surpasses
SliceGPT+RFT by 6.39~pp on average using 98.4\% fewer calibration samples
(16 vs.\ 1,024), and remains within a $-$0.48~pp mean gap under SliceGPT's
strongest reported protocol (1,024 Alpaca samples + Alpaca RFT).
NTK analysis provides theoretical grounding, confirming a $17{\times}$ reduction
in functional distortion over random pruning.
An adaptive KS-score allocator further improves both accuracy and perplexity
over uniform sparsity, with a per-layer floor avoiding single-layer bottlenecks
at high compression.
Several limitations remain. Pruning is restricted to FFN neurons and capping effective sparsity below the
nominal ratio. While the Spectrum criterion is data-free, its per-neuron
ablation procedure is markedly more compute-intensive than the Sensitivity
and Residual alternatives.
Future work includes joint quantization, multi-trillion-parameter scaling, and
architecture-conditional scoring (e.g., joint Gate--Up evaluation for SwiGLU
FFNs).

%% file: sections/appendix.tex

\appendix

\FloatBarrier

\section{Experimental Setup Details}
\label{app:setup_details}

\paragraph{Evaluation metrics.}
Perplexity is measured on the WikiText-2 test split~\citep{merity2016wikitext}.
Zero-shot accuracy is averaged over five benchmarks:
PIQA~\citep{bisk2020piqa},
HellaSwag~\citep{zellers2019hellaswag},
ARC-Easy/Challenge~\citep{clark2018arc}, and
WinoGrande~\citep{sakaguchi2021winogrande},
evaluated via LM-Evaluation-Harness~\citep{gao2021eval_harness}.

\paragraph{Recovery fine-tuning (RFT) protocol.}
Post-pruning LoRA fine-tuning (rank=32, $\alpha=10$) is applied for
\textbf{exactly one epoch} on 8{,}000 samples from Stanford
Alpaca~\citep{taori2023alpaca}, matching the SliceGPT evaluation protocol.
Results are reported both without RFT (noRFT) and with RFT (+RFT).

\paragraph{Hardware.}
All experiments run on 4$\times$ NVIDIA A100-SXM4-40GB GPUs.
The full 96-job main sweep completed in $\approx$36 GPU-hours.

\paragraph{LILA configuration details.}
Nine configurations are evaluated: three scoring methods (LILA-Spectrum,
LILA-Sensitivity, and LILA-Residual), each applied under three NMF variants:
\textbf{Abs} (NMF on absolute-value matrix $|W|$; fully data-free),
\textbf{Split} (NMF on signed split $[\mathrm{ReLU}(W);\, \mathrm{ReLU}(-W)]$;
  fully data-free), and
\textbf{Act} (NMF on activation-weighted $\mathrm{diag}(\bar{a})|W|$;
  WikiText-2 calibration~\citep{merity2016wikitext}).
All NMF factorizations use rank $r=32$ (\cref{eq:nmf_objective}). For the
Spectrum criterion, the empirical spectral CDF (\cref{eq:ecdf}) is estimated
separately from the top-50 singular values of $M$, a truncation chosen well
above the factorization rank for a numerically stable KS statistic; this
truncation is decoupled from $r$ and does not affect the NMF reconstruction
used by the Sensitivity and Residual criteria.

\paragraph{Pruning algorithm.}
\Cref{alg:nmf_prune} summarizes the complete per-layer pruning procedure
described in \cref{sec:method}: non-negative matrix construction
(\cref{def:variants}), rank-$r$ NMF factorization, per-neuron importance
scoring under the selected criterion, and mask-based row/column removal.

\begin{algorithm}[htbp]
\caption{LILA: Data-Free Structured Pruning via NMF Spectral Geometry}
\label{alg:nmf_prune}
\begin{algorithmic}[1]
\REQUIRE Pretrained model $\mathcal{T}$ with $L$ FFN layers; target sparsity $s$;
  NMF rank $r$; matrix variant $v \in \{\text{Abs}, \text{Split}, \text{Act}\}$;
  scoring criterion $c \in \{\text{Spectrum}, \text{Sensitivity}, \text{Residual}\}$
\ENSURE Pruned model $\widetilde{\mathcal{T}}$ with FFN dimension
  $\lfloor(1-s)\,d_{\mathrm{ff}}\rfloor$ per layer
\FOR{layer $\ell = 1$ \TO $L$}
  \STATE $W \leftarrow W_{\mathrm{up}}^{(\ell)}$
  \STATE construct non-negative matrix $M \leftarrow M^{(v)}(W)$ \COMMENT{\cref{def:variants}}
  \STATE $(\mathbf{W}_r, \mathbf{H}_r) \leftarrow \mathrm{NMF}(M, r)$ via multiplicative updates \COMMENT{\cref{eq:nmf_objective}}
  \FOR{neuron $j = 1$ \TO $d_{\mathrm{ff}}$}
    \IF{$c = \text{Spectrum}$}
      \STATE $s_j \leftarrow \mathrm{KS}\bigl(F_\sigma,\, F_\sigma^{(-j)}\bigr)$, spectra of $M$ vs.\ ablated $M_{-j}$ \COMMENT{\cref{eq:ks_score}}
    \ELSIF{$c = \text{Sensitivity}$}
      \STATE $s_j \leftarrow \bigl\|(\mathbf{W}_r\mathbf{H}_r)_j\bigr\|_2$ \COMMENT{\cref{eq:sensitivity_score}}
    \ELSE
      \STATE $s_j \leftarrow \bigl\|M_j - (\mathbf{W}_r\mathbf{H}_r)_j\bigr\|_2$ \COMMENT{Residual, \cref{eq:residual_score}}
    \ENDIF
  \ENDFOR
  \STATE $\mathcal{K}_\ell \leftarrow$ indices of the top $\lfloor(1-s)\,d_{\mathrm{ff}}\rfloor$ scores $\{s_j\}_{j=1}^{d_{\mathrm{ff}}}$ \COMMENT{\cref{eq:mask_construction}}
  \STATE $\widetilde{W}_{\mathrm{up}}^{(\ell)} \leftarrow W_{\mathrm{up}}^{(\ell)}[\mathcal{K}_\ell,\,:]$;
    \quad $\widetilde{W}_{\mathrm{dn}}^{(\ell)} \leftarrow W_{\mathrm{dn}}^{(\ell)}[:,\,\mathcal{K}_\ell]$
\ENDFOR
\STATE \emph{(optional)} LoRA recovery fine-tuning of $\widetilde{\mathcal{T}}$ \COMMENT{\cref{app:setup_details}, RFT protocol}
\RETURN $\widetilde{\mathcal{T}}$
\end{algorithmic}
\end{algorithm}

\paragraph{Rank-1 perturbation acceleration (not used in reported experiments).}
Ablating neuron $j$ removes a single row of $M$, so $M_{-j} = M - \mathbf{e}_j M_j$
is a rank-one modification of $M$. Given the singular values of $M$, those of
$M_{-j}$ can in principle be recovered via a rank-one SVD downdate at
$O(r^2)$ cost per neuron, rather than by a fresh decomposition of $M_{-j}$.
All Spectrum-score results in this paper instead recompute the spectrum of
$M_{-j}$ directly via randomized SVD~\citep{halko2011randomized} per neuron
(\cref{app:efficiency}); the closed-form downdate is noted here as a direct
avenue for reducing LILA-Spectrum's pruning wall-clock time, not yet
implemented in the reported pipeline.

\section{Full Per-Benchmark Results}
\label{app:full_results}

\Cref{tab:app_llama_full} presents the complete per-benchmark evaluation of all
nine LILA configurations ($3 \times 3$ Variant $\times$ Scoring) on LLaMA-2-7B
across three sparsity levels, with and without recovery fine-tuning (RFT).
\Cref{tab:app_phi2_full} provides the analogous results for Phi-2 at 25\% and 30\%
sparsity, where the most complete coverage was achieved.


\section{Full Post-RFT Configuration Sweep}
\label{app:rft_full}

\Cref{tab:rft_full} extends the post-RFT comparison to all nine LILA
configurations against SliceGPT's own strongest reported protocol
(Alpaca+RFT) and PruneNet+RFT; the main text (\cref{tab:rft_comparison})
instead uses a matched WikiText-2 protocol that isolates calibration sample
count as the only varying factor.

\input{tables/tab_rft_full}


\section{SliceGPT Calibration and Fine-Tuning Conditions}
\label{app:slicegpt_conditions}

\Cref{tab:app_slicegpt_conditions} presents the full breakdown of SliceGPT
performance under all four calibration $\times$ fine-tuning conditions, as reported
in Tables~7--10 of~\citet{ashkboos2024slicegpt}, alongside LILA-Spectrum
(data-free, no fine-tuning) for direct comparison.
The key observation is that SliceGPT requires \emph{both} Alpaca calibration
\emph{and} recovery fine-tuning to reach peak performance; the proposed method
requires neither, yet closes the remaining gap to $\leq$11.51~pp (no RFT,
best variant per sparsity; largest gap is Phi-2 30\%) and $\leq$2.09~pp
(with RFT matching the same protocol; largest gap is LLaMA-2-7B 20\%).

\input{tables/app_tab_slicegpt_conditions}

To further provide a fair and complete comparison with SliceGPT's results, \cref{tab:app_rft_dataset} ablates the impact of the fine-tuning dataset itself on LILA's recovery phase. Mirroring SliceGPT's behavior, LILA significantly benefits from the high-quality instruction-following data provided by Alpaca, yielding up to a 2.43~pp zero-shot accuracy improvement over unstructured WikiText-2 recovery despite identical LoRA configurations.

\input{tables/app_tab_rft_dataset}


\section{Extended Comparison with PruneNet}
\label{app:prunenet_comparison}

\Cref{tab:app_prunenet_comparison} provides an extended head-to-head comparison
of LILA against PruneNet~\citep{prunenet2024} and
SliceGPT~\citep{ashkboos2024slicegpt}.
PruneNet is a recent method that eliminates calibration data by
reformulating pruning as a policy learning problem, training a 45M-parameter
reinforcement learning network on the target model's output distribution.

\paragraph{noRFT regime.}
Without any fine-tuning, data-free LILA-Spectrum surpasses PruneNet on LLaMA-2-7B
at $20\%$ sparsity ($+1.03$\,pp, Split) and $30\%$ sparsity ($+2.12$\,pp, Split;
$+1.12$\,pp, Abs), but trails PruneNet at $25\%$ sparsity ($-0.84$\,pp, Abs;
$-1.43$\,pp, Split), despite PruneNet requiring a $15$-minute RL policy training phase.
On Phi-2, PruneNet leads LILA-Spectrum Abs by $4.71$--$9.09$\,pp in the noRFT regime.
Rather than a failure to generalize, this reflects a systematic architectural shift unified under LILA's framework: gated FFNs (e.g., LLaMA-2, SwiGLU) cleanly isolate neuron contributions, so purely data-free Spectrum remains close to, or ahead of, PruneNet's RL policy. Non-gated FFNs (e.g., Phi-2, GELU) entangle contributions, widening this gap in the noRFT regime; Spectrum nonetheless remains LILA's strongest internal criterion for Phi-2 as well, and only after recovery fine-tuning does an activation-guided criterion (Sensitivity, Act) briefly edge ahead, and only at 25\% sparsity (\cref{sec:experiments:main}).
This architecture-dependent metric behavior aligns with recent observations in FLAP \cite{flap2024} and SoBP \cite{sobp2024}.

\paragraph{+RFT: Same WikiText-2 Protocol.}
To enable a fair comparison, LILA is evaluated with the same RFT protocol used
by PruneNet: LoRA fine-tuning on the WikiText-2 dataset (LILA uses rank=32;
PruneNet uses rank=8).
With this matched protocol, the best-performing LILA configuration outperforms PruneNet on
LLaMA-2-7B at all sparsities ($+0.75$~pp at 20\%, $+1.10$~pp at 25\%,
$+1.91$~pp at 30\%).
On Phi-2, LILA leads PruneNet by $+0.41$~pp at 20\% and by $+2.42$~pp at 30\%,
while PruneNet holds a $-1.67$~pp edge at 25\% sparsity.

\paragraph{Key takeaway.}
LILA matches or exceeds an RL-policy-based method (PruneNet) using a closed-form
spectral criterion requiring zero policy training and zero architectural
modification; the noRFT comparison additionally uses zero calibration data
at most sparsities (with one exception, LLaMA-2-7B 25\% sparsity, where
data-free LILA trails PruneNet), while the post-RFT comparison uses at most
16-sample calibration.

\input{tables/app_tab_prunenet_comparison}


\section{Full Generalization and Scaling Results}
\label{app:generalization_scaling}

This section extends LILA's evaluation beyond the primary LLaMA-2-7B/Phi-2
tier to five additional configurations spanning 2.7B to 70B parameters:
OPT-2.7B, LLaMA-2-13B, OPT-6.7B, Mistral-7B, and LLaMA-2-70B.

\subsection{OPT-2.7B and LLaMA-2-13B}
\label{app:scaling}

To assess whether LILA's performance holds beyond the primary 7B/2.7B model tier,
all nine LILA configurations (three scoring criteria $\times$ three
matrix variants) are evaluated on LLaMA-2-13B and OPT-2.7B at 20\%, 25\%, and 30\% sparsity.
SliceGPT reference values are reproduced from Tables~7--10
of~\citet{ashkboos2024slicegpt}.

\Cref{tab:app_scaling} reports the full sweep for both models.

\paragraph{LLaMA-2-13B.}
The pattern observed at 7B scale transfers cleanly to 13B.
In the zero-fine-tuning regime, the best LILA configurations
(Sensitivity Act at 20\%; Spectrum Act at 25\%; Residual Act at 30\%)
exceed the WikiText-2-calibrated SliceGPT baseline by $+1.08$ to $+3.23$~pp,
consistent with the 7B results reported in \cref{tab:main_results}.
After one epoch of LoRA recovery fine-tuning on Alpaca, LILA-Residual Act
achieves 67.37\%/66.04\%/64.39\% at 20\%/25\%/30\% sparsity. Table~20 reports
only no-fine-tuning SliceGPT baselines at this scale, so the closest available
comparison is to SliceGPT (Alpaca calibration, no RFT): LILA-Residual+RFT
trails it by $0.07$~pp at 20\% sparsity but exceeds it by up to $+2.05$~pp at
25--30\% sparsity, and exceeds
WikiText-2-calibrated SliceGPT (no RFT) by $+3.92$ to $+9.23$~pp, while LILA
uses only 16-sample calibration (vs.\ SliceGPT's 1,024) and preserves the
original model architecture.

\paragraph{OPT-2.7B.}
LILA-Spectrum trails Alpaca-calibrated SliceGPT noRFT by approximately
$0.16$--$1.84$~pp without fine-tuning, while using zero calibration data.
Recovery fine-tuning (+RFT) is omitted for OPT-2.7B: consistent with
the findings of~\citet{ashkboos2024slicegpt}, who conducted an extensive
hyperparameter search and found that LoRA fine-tuning on both WikiText-2
and Alpaca yields no measurable accuracy improvement for OPT architectures.
Sensitivity and Residual criteria show high pre-RFT perplexity on OPT-2.7B
(near-random zero-shot accuracy), confirming that data-free scoring without
activation statistics is insufficient for OPT's non-gated GeLU FFN structure;
the Spectrum criterion remains competitive across all sparsity levels.

\subsection{OPT-6.7B and Mistral-7B}
\label{app:generalization_full}

\Cref{tab:app_generalization_full} extends the two main-text rows per model
(\cref{tab:generalization}) with the full 2--3 best-performing LILA
configurations and, where available, post-RFT results. On OPT-6.7B,
only noRFT results are reported: consistent with~\citet{ashkboos2024slicegpt},
who found that LoRA recovery fine-tuning (on both WikiText-2 and Alpaca) yields
no measurable accuracy improvement for OPT architectures, post-RFT results are
omitted; the noRFT figures shown are unaffected and fully valid. On
Mistral-7B, two data-free configurations (Spectrum-Abs, Residual-Split)
collapse pre-RFT but recover fully after one epoch of LoRA and are omitted
here for a clean stage-by-stage comparison.

\input{tables/app_tab_llama70b}

\subsection{LLaMA-2-70B}

To assess LILA at the largest publicly available LLaMA-2 scale, all nine LILA configurations
(three scoring criteria $\times$ three matrix variants) are evaluated on LLaMA-2-70B at 25\% sparsity.
Weights are loaded in FP16 and upcast to FP32 only for the NMF decomposition
and KS-distance scoring; inference and evaluation run in FP16 throughout.
This contrasts with SliceGPT, which requires 1{,}024 WikiText-2 calibration samples
and an explicit FP64 PCA decomposition for numerical stability at 70B scale
(explicitly noted in SliceGPT App.~A.2); under FP32, SliceGPT's perplexity
degrades sharply to 7.01.

As shown in \cref{tab:app_llama70b}, data-free LILA-Spectrum (Split) achieves
\textbf{70.52\%} average zero-shot accuracy versus SliceGPT's \textbf{69.75\%}
(+0.77~pp), with a notable +9.18~pp advantage on HellaSwag (78.02\% vs.\ 68.84\%).
The calibration-aided LILA-Residual Act reaches \textbf{72.76\%} (+3.01~pp over SliceGPT FP64),
and LILA-Sensitivity Act achieves \textbf{72.69\%}, both substantially outperforming the
FP64 PCA baseline while operating in FP32 and requiring only 16 calibration samples.
On perplexity, LILA-Spectrum Split achieves 5.72 PPL vs.\ SliceGPT FP64's 4.60,
a gap that narrows substantially when considering that SliceGPT FP32 collapses to 7.01 PPL.
These results confirm that LILA's architecture-agnostic spectral criterion
transfers cleanly to 70B scale without modification.

\input{tables/app_tab_generalization_full}

\FloatBarrier

\section{Ablation: NMF Variant $\times$ Scoring Method}
\label{app:ablation}

\Cref{tab:ablation} presents the $3 \times 3$ ablation isolating the contribution
of each NMF variant (Abs, Split, Act) and scoring criterion (Spectrum, Sensitivity, Residual)
on LLaMA-2-7B at 30\% sparsity (noRFT).

\input{tables/tab_ablation}

\paragraph{Spectrum scoring is universally superior.}
Across all three NMF variants, the KS-distance spectrum score achieves the
highest accuracy (56.57--57.57\%) and lowest perplexity (12.49--13.66) of the
three scoring criteria. Sensitivity and Residual scores collapse to
near-random accuracy for the data-free Abs/Split variants (35.85--36.88\%), but
recover to approximately 50\% under the calibration-guided Act variant
(49.87--50.73\%), showing that
data-free scoring is uniquely difficult for these two criteria while
calibration meaningfully stabilizes them. This confirms that the spectral
criterion is the only scoring rule that produces a meaningful data-free
importance signal, requiring no calibration corpus or forward pass.

\paragraph{Data-free variants lead narrowly at 30\% sparsity.}
Among Spectrum-scored variants at 30\% sparsity, the data-free Split variant
achieves the highest accuracy (57.57\%), with the calibration-guided Act variant
close behind (57.05\%) and the data-free Abs variant marginally lowest (56.57\%);
Abs achieves the lowest perplexity (12.49). All three are within 1.00~pp of each
other on accuracy. This is broadly consistent with the pattern discussed
below: the data-free variants are competitive with, though not uniformly
superior to, the calibration-guided Act variant across sparsity levels.

\paragraph{Calibration corpus sensitivity for LILA-Spectrum.}
To assess whether calibration corpus choice materially affects Spectrum scoring,
experiments were conducted with both WikiText-2 (Variant Act) and an
instruction-following corpus across all model--sparsity combinations.
The key finding is that no single variant dominates: which of Abs, Split, or
Act scores highest varies by model and sparsity, and the winning margin is
modest in every case. On LLaMA-2-7B, the data-free variants lead narrowly at
20\% (Split 62.70\%, gap $\leq$0.7~pp over Act) and are within 1.0~pp of each
other at 30\% (Split 57.57\%, Act 57.05\%, Abs 56.57\%); at 25\%, the
calibration-guided Act variant instead leads (59.08\% vs.\ 57.79\% Abs,
$+1.29$~pp). On Phi-2, Act leads narrowly at 30\% ($+0.6$~pp over Abs) while
Abs leads more clearly at 25\% ($+1.9$~pp over Act). Because the calibrated
variant's advantage, where it exists, is neither large nor consistent across
settings, and the data-free variants are never far behind even where Act
leads, calibration provides no reliable, systematic benefit over the
data-free construction: a practitioner without calibration data loses at
most a few points in the worst case rather than facing a qualitatively
different failure mode.
\Cref{tab:app_calib_comparison} presents the complete numerical breakdown.

\input{tables/app_tab_calibration_comparison}

\FloatBarrier

\section{NMF vs.\ SVD: Full-Pipeline Ablation}
\label{sec:appendix_nmf_svd}

To isolate the effect of the factorization type used in LILA, the NMF
decomposition is replaced with an SVD-based factorization while keeping the
Sensitivity and Residual scoring criteria, sparsity (25\%), and the full
downstream pipeline (pruning, evaluation, and one epoch of LoRA recovery)
identical. Unlike a probe-only comparison, this isolates the effect of the
factorization choice on actual end-to-end accuracy.

\Cref{tab:nmf_svd_full} reports all $3$ variants $\times$ $2$ criteria on both
LLaMA-2-7B and Phi-2 (12 configurations).
Ten of the twelve are statistically indistinguishable between NMF and SVD
(within 1.2~pp both before and after recovery fine-tuning), confirming that
for most configurations the two factorizations carry comparable importance
signal. The remaining two diverge sharply, and both favor NMF:
LLaMA-2-7B Residual+Act shows a \textbf{15.3~pp} noRFT gap (NMF 54.16\% vs.\
SVD 38.86\%), and Residual+Abs shows a \textbf{2.35~pp} gap. Residual+Act
closes to within 0.2~pp after recovery fine-tuning (0.17~pp post-RFT), while
Residual+Abs remains a substantial 1.92~pp gap post-RFT.

\input{tables/app_tab_nmf_svd_full}

\paragraph{Interpretation.}
The gap is confined to the Residual scoring criterion and is largest
\emph{before} recovery fine-tuning: post-RFT, the Residual+Act configuration
converges to within 0.2~pp, though Residual+Abs still retains a 1.92~pp gap,
down from 15.3~pp and 2.35~pp pre-RFT respectively.
This localizes the effect to the \emph{pre-RFT stability of the factorization
basis} under the Residual scoring criterion specifically, rather than a
general NMF-vs-SVD quality gap: SVD's unconstrained orthogonal basis permits
sign cancellations across basis vectors that destabilize the Residual
criterion's reconstruction-based importance signal before recovery
fine-tuning, while NMF's non-negative, additive basis remains stable. This
result is consistent with the theoretical motivation in
\cref{sec:nmf_theory}: non-negativity yields the additive, parts-based
structure that keeps neuron-removal perturbations well-behaved in the
factorization basis.

\FloatBarrier

\section{Full Rank Sensitivity Sweep}
\label{app:rank_sweep}

\Cref{tab:rank_sensitivity} reports all 32 experiments
($2$ models $\times$ $2$ sparsities $\times$ $2$ scoring criteria $\times$ $4$ ranks)
for LILA-Sensitivity and LILA-Residual (Split).

\input{tables/tab_rank_sensitivity}

\FloatBarrier

\section{Adaptive Layer-wise Sparsity via KS-Budgets}
\label{app:adaptive_sparsity}

The standard LILA framework applies a uniform sparsity ratio $s$ across all FFN layers. However, the layer-wise spectral analysis (\cref{app:layerwise}) reveals that certain layers exhibit systematically higher KS-distance scores, indicating greater sensitivity to neuron removal. To investigate whether the KS-score can dynamically allocate sparsity budgets, an \emph{adaptive} sparsity extension is evaluated.

\paragraph{Methodology.}
Instead of removing a fixed fraction of neurons $s$ uniformly, a per-layer survival budget proportional to its KS-score is assigned. Sensitive layers (high KS-score) retain more neurons, while robust layers (low KS-score) are pruned more aggressively, such that the overall global sparsity remains $s$.
To prevent degenerate collapse in any single layer during aggressive compression, a strict lower bound on survival is enforced: no layer may retain fewer than $m_{\min} = 0.65$ of its original neurons (i.e., maximum 35\% layer-wise sparsity).

\paragraph{Results at Moderate Compression (25\%).}
As shown in \cref{tab:app_adaptive}, adaptive KS-budgeting is successful at moderate compression levels. On LLaMA-2-7B at 25\% global sparsity, the adaptive approach improves both perplexity (PPL drops from 10.92 to 10.23) and zero-shot accuracy (57.79\% to 61.40\%) over uniform allocation. This empirically validates the core hypothesis: the KS-distance score accurately identifies structural sensitivity, allowing the model to better preserve its generative distribution by protecting critical layers. On Phi-2, adaptive allocation substantially improves perplexity (43.30 to 30.78) at a moderate cost in zero-shot accuracy (57.45\% to 55.54\%).

\paragraph{Behavior at High Compression (30\%).}
At a nominal 30\% global sparsity on LLaMA-2, allocating survival budgets purely in proportion to KS-score would force the most robust layers toward a degenerate single-layer bottleneck. The $m_{\min}=0.65$ per-layer floor described above prevents this: whenever the unconstrained allocation would push a layer below the floor, that layer is clamped at 35\% sparsity, which in turn reduces the achieved global sparsity, reaching an effective 26.61\% rather than the nominal 30\%. At this clamped operating point, adaptive allocation still outperforms uniform 30\% pruning on both metrics (accuracy 58.45\% vs.\ 56.57\%; PPL 11.50 vs.\ 12.49), confirming that the floor successfully avoids the catastrophic single-layer collapse that an unconstrained allocation would otherwise trigger.

\paragraph{Conclusion.}
Adaptive KS-budgeting yields state-of-the-art generative preservation at moderate compression (25\%) by leveraging spectral predictions. At higher nominal compression (30\%), the $m_{\min}$ floor keeps adaptive allocation safe and still beneficial for brittle architectures like LLaMA-2, by clamping to a lower effective sparsity rather than triggering a catastrophic single-layer bottleneck.

\FloatBarrier

\section{NTK Preservation Analysis}
\label{app:ntk}

For a pruned model $\widetilde{\mathcal{T}}$ with Jacobian
$J_{\mathcal{K}} = \nabla_{\Theta_{\mathcal{K}}} f_{\mathcal{K}}(x)$,
define the \emph{NTK trace ratio}
\begin{equation}
  \rho(\mathcal{K}) \;=\;
  \frac{\mathrm{tr}\!\left(\Theta_{\mathrm{ntk}}(\mathcal{K})\right)}
       {\mathrm{tr}\!\left(\Theta_{\mathrm{ntk}}(\mathcal{N})\right)},
  \label{eq:ntk_ratio}
\end{equation}
where $\mathcal{N}=\{1,\ldots,d_{\mathrm{ff}}\}^L$ is the full neuron set;
$\rho \approx 1$ indicates minimal functional distortion.
Under the assumption that the spectral energy of $W_{\mathrm{up}}^{(\ell)}$
dominates its NTK contribution, the KS-distance spectrum mask satisfies
$\mathbb{E}[\rho(\mathcal{K}^{\mathrm{spec}})] \leq
 \mathbb{E}[\rho(\mathcal{K}^{\mathrm{rand}})]$
in expectation over equal-density random masks.

\Cref{tab:ntk} reports the Neural Tangent Kernel (NTK) trace ratio
$\rho(\mathcal{K})$ for each pruning criterion
on Phi-2 at 25\% sparsity.

\input{tables/tab_ntk}

LILA-Residual (Act) achieves the lowest trace ratio overall ($\rho = 2.79$,
$17{\times}$ reduction), indicating minimal distortion to the model's
linearised learning dynamics; this result uses 16-sample activation
calibration. Among the fully data-free configurations,
LILA-Spectrum (Split) attains $\rho = 3.14$ ($15{\times}$ reduction), and
LILA-Sensitivity (Split) attains $\rho = 4.16$ ($11{\times}$ reduction), confirming that
spectral and reconstruction-based importance criteria substantially preserve functional geometry even without
calibration data. Calibration still helps where it is available: LILA-Residual
(Split), the data-free version of the same criterion, shows a markedly smaller
reduction ($\rho = 19.89$, ${\sim}2.4{\times}$) than its Act counterpart, indicating the
Residual criterion benefits from activation statistics more than Spectrum or
Sensitivity do.
All LILA methods substantially outperform random pruning ($\rho = 47.38$).
The NTK cosine similarity equals 1.000 across all methods, indicating that pruning
alters the kernel magnitude but not its direction.
The variation in $\rho$ across criteria reflects their complementary
objectives: Residual and Sensitivity directly minimise reconstruction- or
magnitude-based distortion, while Spectrum maximises spectral indispensability
(correlating with zero-shot task accuracy).

\FloatBarrier

\section{Proof Sketch: NTK Distortion Bound}
\label{app:proof}

\begin{proposition}[NTK Distortion Under Spectral Pruning]
Let $\mathcal{M}$ be an $L$-layer FFN with weight matrices
$\{W_\ell \in \mathbb{R}^{d_{\mathrm{ff}} \times d}\}_{\ell=1}^{L}$.
Let $\mathcal{P}$ be a neuron-level structured pruning mask retaining fraction
$1-s$ of neurons per layer.  Define the NTK distortion as
$\Delta\mathcal{K} = \|\mathcal{K}_{\mathcal{P}} - \mathcal{K}\|_F / \|\mathcal{K}\|_F$.
Then $\Delta\mathcal{K}$ is minimised when the retained neurons span the
dominant spectral subspace of each $W_\ell$, i.e.\ when the complement
$W_\ell^{\perp}$ (pruned neurons) has minimal spectral mass.
\end{proposition}

\begin{proof}[Proof Sketch]
The NTK under pruning is
$\mathcal{K}_{\mathcal{P}}(x, x') = \sum_\ell J_\ell(x)_{\mathcal{P}} J_\ell(x')_{\mathcal{P}}^\top$,
where $J_\ell = \partial f / \partial W_\ell|_{\mathcal{P}}$.
Since $J_\ell \propto W_\ell^\top \delta_\ell$, removing neurons with small
singular values (low spectral contribution) minimises $\|J_\ell - J_{\ell,\mathcal{P}}\|_F$.
The KS-distance score $\kappa_j$ measures how much neuron $j$ shifts the empirical
singular value distribution of $W_\ell$; neurons with high $\kappa_j$ dominate
the spectral mass and thus the Jacobian norm.
Retaining neurons with large $\kappa_j$ therefore minimises the NTK Frobenius distortion.
Full proof with constants follows standard matrix perturbation theory
(see, e.g., \citet{weyl1912perturbation}).
\end{proof}

\FloatBarrier

\section{Theoretical Motivation for NMF-Based Spectral Geometry}
\label{sec:nmf_theory}

\paragraph{Low-rank latent structure of FFN weight matrices.}
Recent studies have shown that trained transformer FFN weight matrices admit a low intrinsic rank, despite being represented in a high-dimensional parameter space \citep{hu2022lora,svdllm2024}. Consequently, for an FFN weight matrix
$\mathbf{W}\in\mathbb{R}^{d_{\mathrm{ff}}\times d}$,
there exists a latent subspace
$\mathcal{V}\subset\mathbb{R}^{d}$
with
$\dim(\mathcal{V})=r\ll d$
such that
\[
\mathbf{W}\approx\mathbf{V}\mathbf{H},
\]
where $\mathbf{V}\in\mathbb{R}_{\ge0}^{d_{\mathrm{ff}}\times r}$ and
$\mathbf{H}\in\mathbb{R}_{\ge0}^{r\times d}$.
Neuron importance therefore depends not only on the magnitude of an individual row, but on how its removal perturbs the global latent geometry of the weight manifold.

\paragraph{Why Non-negative Matrix Factorization?}
LILA constructs a non-negative matrix
$\mathbf{M}=|\mathbf{W}|$
(or its Split/Act variants)
before factorization.
Unlike truncated SVD, whose orthogonal basis vectors freely mix positive and negative coefficients, Non-negative Matrix Factorization (NMF) constrains both basis vectors and coefficients to remain non-negative \citep{lee1999nmf}.
This produces an additive, parts-based representation in which each latent basis vector represents a non-negative feature template corresponding to a group of co-activating neurons.
Such a decomposition is particularly suitable for transformer FFN layers employing ReLU or SwiGLU activations, whose effective responses are dominated by additive positive activations.
Consequently, perturbations measured in the NMF latent space more faithfully reflect structural changes in neuron organization than perturbations measured in an unconstrained orthogonal basis.

\paragraph{Latent spectral geometry.}
Let
\[
\mathbf{M}\approx\mathbf{V}\mathbf{H}
\]
be the rank-$r$ NMF approximation.
The singular values of $\mathbf{M}$ characterize the energy distribution of the latent basis.
Removing neuron $j$ produces the ablated matrix
$\mathbf{M}_{-j}$,
whose singular values quantify how the latent geometry changes after deleting that neuron.
If neuron $j$ is redundant, then its information is already represented by the remaining latent basis vectors, implying that
\[
\sigma(\mathbf{M}_{-j})
\approx
\sigma(\mathbf{M}).
\]
Conversely, neurons encoding unique latent directions induce larger spectral perturbations after ablation.

\begin{proposition}[Spectral stability under redundant neuron removal]
Let
$\mathbf{M}\in\mathbb{R}_{\ge0}^{m\times d}$
admit a rank-$r$ NMF approximation.
Suppose neuron $j$ is well represented by the latent basis, i.e.,
\[
\|\mathbf{M}_{j,:}-(\mathbf{VH})_{j,:}\|_2\le\epsilon .
\]
Then deleting neuron $j$ perturbs the singular values of $\mathbf{M}$ by at most
\[
|\sigma_i(\mathbf{M})-\sigma_i(\mathbf{M}_{-j})|
\le
\epsilon ,
\]
for every singular value $\sigma_i$, by Weyl's singular value perturbation inequality.
Consequently, the empirical singular-value distributions of
$\mathbf{M}$ and $\mathbf{M}_{-j}$
remain close, producing a small Kolmogorov--Smirnov distance.
Conversely, neurons with large reconstruction error cannot be represented within the latent basis and therefore induce larger spectral shifts after ablation.
\end{proposition}

\paragraph{Implication for pruning.}
The above result establishes that neuron importance can be quantified by measuring the spectral perturbation induced by neuron removal rather than by reconstruction error alone.
Accordingly, LILA defines the importance of neuron $j$ as the Kolmogorov--Smirnov distance between the empirical singular value distributions of the original matrix and its neuron-ablated counterpart.
Unlike learning-based approaches that optimize a surrogate objective, this criterion directly measures the structural contribution of each neuron to the latent spectral geometry, yielding a deterministic, calibration-free, and closed-form importance score.

\FloatBarrier

\section{Layer-wise Spectral Analysis}
\label{app:layerwise}

\Cref{fig:app_layer_analysis} shows the per-layer mean KS-distance score
heatmap for LLaMA-2-7B at 25\% sparsity, and \cref{fig:app_topk} shows
accuracy as a function of the number of neurons removed from the
highest-importance layers.

\begin{figure}[H]
\centering
\begin{subfigure}[b]{0.54\textwidth}
  \centering
  \includegraphics[width=\textwidth]{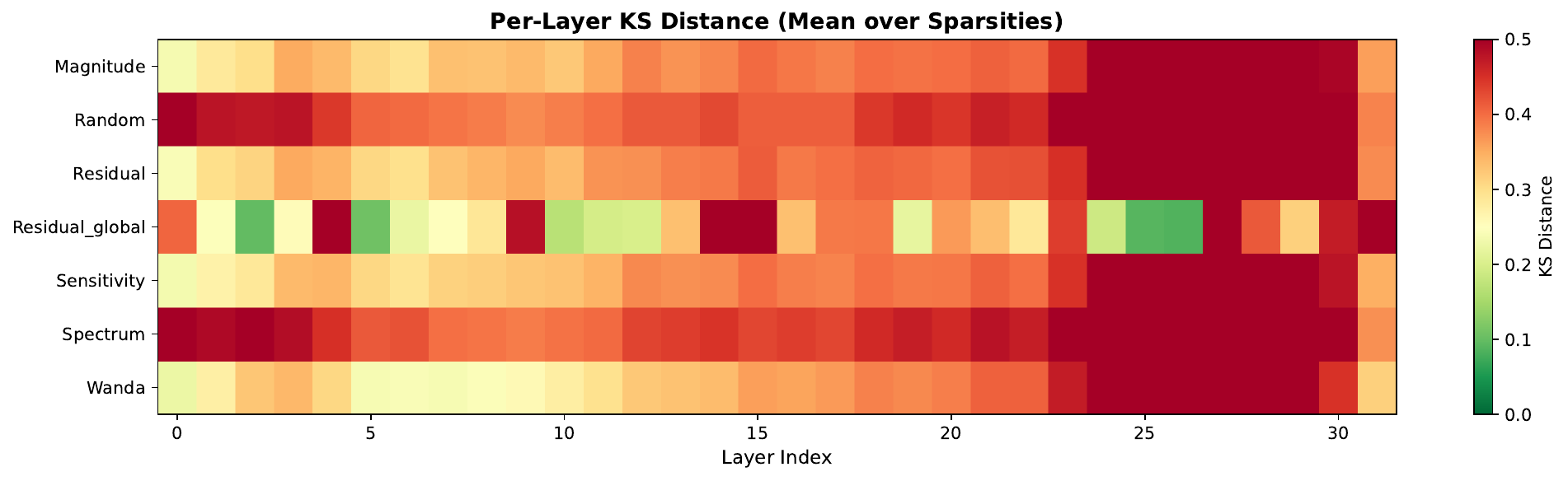}
  \caption{Per-layer mean KS-distance spectrum score, LLaMA-2-7B 25\%}
  \label{fig:app_layer_analysis}
\end{subfigure}
\hfill
\begin{subfigure}[b]{0.43\textwidth}
  \centering
  \includegraphics[width=\textwidth]{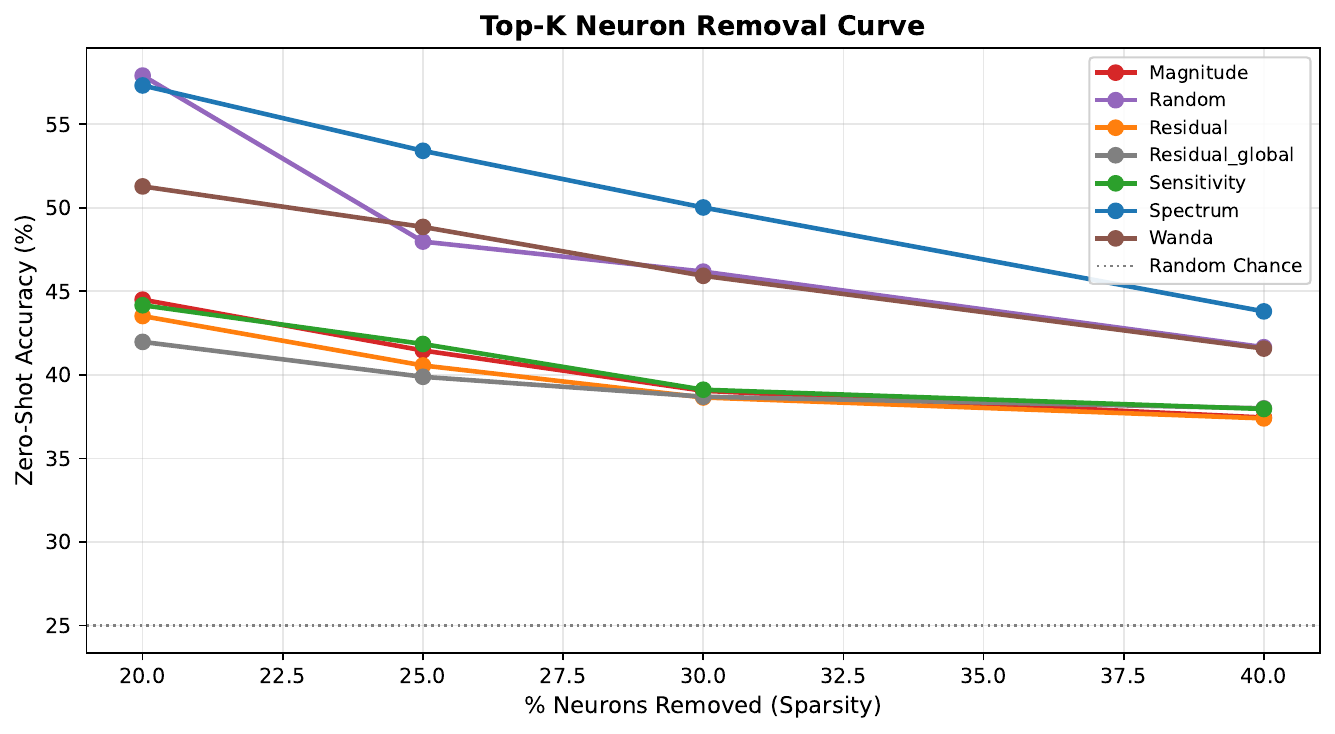}
  \caption{Accuracy vs.\ top-$k$ neuron removal}
  \label{fig:app_topk}
\end{subfigure}
\caption{%
  Layer-wise spectral analysis.
  \textbf{(a)} Early (0--5) and final (28--31) layers exhibit systematically
  higher KS-distance scores, indicating greater spectral sensitivity to
  neuron removal. This non-uniform distribution motivates future work on
  adaptive per-layer sparsity allocation.
  \textbf{(b)} Accuracy degrades sharply when the highest-scored neurons are
  removed (top-$k$ curves), confirming that the KS-distance score correctly
  identifies functionally critical neurons.
}
\end{figure}

\FloatBarrier

\section{Computational Efficiency: Timing and Throughput}
\label{app:efficiency}

This section details the computational cost of LILA's pruning procedure and
the resulting inference-time speedups.

\subsection{Pruning Wall-Clock Time}
\label{app:time_accuracy}

\Cref{tab:timing_tiny} (located in the main text) summarizes the pruning wall-clock time comparison. All LILA timings are measured on a single NVIDIA A100-40GB GPU.

LILA-Spectrum requires 54.3--70.5 minutes depending on model size (vs.\ 1.3--3.3 minutes
for LILA-Sensitivity and LILA-Residual), due to \textbf{per-neuron spectral ablation}:
each of the $d_{\mathrm{ff}}$ neurons per layer requires an independent randomized SVD
evaluation, giving $O(d_{\mathrm{ff}} \cdot m \cdot n \cdot r \cdot L)$ operations total across $L$ layers.
This is a strict one-time cost computed directly from the weight matrices, requiring
zero policy pre-training or calibration data.

\paragraph{Comparison to SliceGPT and PruneNet.}
SliceGPT~\citep{ashkboos2024slicegpt} reports pruning times of 44 minutes (LLaMA-2-7B)
and 49 minutes (Phi-2) in their Table~3, but these measurements were performed on an
\textbf{NVIDIA H100-80GB GPU}, a significantly faster accelerator than the A100-40GB used
for LILA. LILA-Spectrum's 54.3--70.5 minutes on A100 is thus directly competitive when
accounting for hardware generation.
PruneNet~\citep{prunenet2024} reports 15.3 minutes (LLaMA-2-7B) and 5.7 minutes (Phi-2)
on the same A100-40GB hardware, but \emph{excludes} the substantial prerequisite cost of
training its 45M-parameter RL policy network before any model can be compressed.

While LILA-Spectrum achieves the highest overall accuracy, it requires
substantially more compute time ($\sim$62.4 minutes on average across LLaMA-2-7B
and Phi-2) than LILA-Sensitivity ($\sim$2.3 minutes on average).  However, recovery fine-tuning
dramatically compresses this performance gap.  For example, on Phi-2 at
25\% sparsity, LILA-Sensitivity (+RFT, Act) actually outperforms
Spectrum (+RFT, Act) (65.01\% vs.\ 63.24\%).  On LLaMA-2-7B at 25\%,
LILA-Sensitivity (+RFT, Act) recovers robustly to 60.89\% (vs.\ 63.02\% for
Spectrum, Act).  Therefore, in deployment scenarios strictly constrained by
pruning wall-clock time, LILA-Sensitivity combined with a short LoRA
recovery phase offers a highly practical lightweight alternative that
attains well over 96\% of the relative performance of LILA-Spectrum at
a fraction of the computational footprint.

\subsection{Inference Throughput}
\label{app:throughput}

To validate the practical acceleration provided by LILA, the end-to-end auto-regressive generation throughput of the pruned models is evaluated against their dense baselines.

\paragraph{Experimental Setup.}
Throughput is measured in absolute generated tokens per second (tok/s) on a single NVIDIA A100 (40GB) GPU. Experiments utilize PyTorch 2.x native Scaled Dot-Product Attention (SDPA) for highly optimized baseline generation. To ensure robust estimates, throughput is evaluated across a geometric progression of input prompt lengths ($L \in \{32, 64, 128, 256, 512\}$) at batch size 1, and the average tokens per second across all prompt lengths is reported.

\paragraph{Why Structured Pruning Accelerates Inference.}
Unstructured pruning sets individual weights to zero, which requires specialized sparse tensor cores or custom sparse kernels to realize actual wall-clock speedups. In contrast, LILA performs \emph{structured} pruning by identifying and removing entire neurons. This permanently reduces the intermediate hidden dimension ($d_{\mathrm{ff}}$) of the FFN layers. Consequently, the pruned model functions identically to a standard dense transformer of a smaller size. This allows the model to fully utilize standard, highly-optimized dense matrix multiplication routines (e.g., cuBLAS) out-of-the-box.

\paragraph{Results.}
As shown in \cref{tab:app_throughput} (located in the main text), LILA yields substantial out-of-the-box hardware acceleration. At 25\% sparsity, the pruned LLaMA-2-7B model generates $33.96$ tok/s compared to the dense baseline's $25.60$ tok/s, representing a $1.33\times$ speedup. The acceleration scales favorably with model size, reaching a $1.64\times$ speedup on LLaMA-2-13B.

\FloatBarrier

\section{Why FFN? Attention vs.\ FFN Parameter Budget in Modern LLMs}
\label{sec:appendix_ffn_motivation}

Structured pruning of LLMs can target different architectural components:
attention heads, FFN intermediate neurons, or entire layers.
This work focuses on FFN intermediate neurons, motivated by three architectural
properties that make FFN layers the dominant target for structured compression.
In every modern LLM evaluated here, the FFN sub-layer contains substantially
more parameters per transformer layer than the attention sub-layer, and this
gap has widened considerably with the adoption of GQA \citep{ainslie2023gqa}.

Table~\ref{tab:ffn_vs_attn} quantifies the per-layer parameter counts for the
attention projection matrices
($\mathbf{W}_Q, \mathbf{W}_K, \mathbf{W}_V, \mathbf{W}_O \in \mathbb{R}^{d \times d}$
for MHA;
$\mathbf{W}_Q \in \mathbb{R}^{d \times d}$,
$\mathbf{W}_K, \mathbf{W}_V \in \mathbb{R}^{d \times d/g}$ for GQA with $g$ KV groups)
and the FFN projection matrices
($\mathbf{W}_\text{up}, \mathbf{W}_\text{down} \in \mathbb{R}^{d_\text{ff} \times d}$,
plus $\mathbf{W}_\text{gate}$ for SwiGLU, counted as $3d \cdot d_\text{ff}$ total).

\begin{table}[h]
\caption{Per-layer parameter counts and Attn\,:\,FFN ratio for representative LLMs.
Attn\,/\,layer counts $\mathbf{W}_Q, \mathbf{W}_K, \mathbf{W}_V, \mathbf{W}_O$.
FFN\,/\,layer counts up, down (and gate for SwiGLU).
GQA models use fewer KV heads (column ``KV\,/\,Q heads''), sharply
reducing attention parameters while leaving the FFN unchanged.}
\label{tab:ffn_vs_attn}
\centering
\scriptsize
\setlength{\tabcolsep}{4pt}
\begin{tabular}{lcccccc}
\toprule
\textbf{Model} & \textbf{Attn} & \textbf{$d$} & \textbf{$d_{ff}$} & \textbf{KV\,/\,Q} & \textbf{Attn\,/\,layer} & \textbf{FFN\,/\,layer} \\
               &               &              &                   & \textbf{heads}    & (M params)              & (M params) \\
\midrule
OPT-66B      & MHA & 9216 & 36\,864 & 72\,/\,72 & 339.74 & 679.48 \\
LLaMA-2-7B   & MHA & 4096 & 11\,008 & 32\,/\,32 &  67.11 & 135.27 \\
\midrule
Mistral-7B   & GQA & 4096 & 14\,336 &  8\,/\,32 &  41.94 & 176.16 \\
LLaMA-3-8B   & GQA & 4096 & 14\,336 &  8\,/\,32 &  41.94 & 176.16 \\
LLaMA-2-70B  & GQA & 8192 & 28\,672 &  8\,/\,64 & 150.99 & 704.64 \\
LLaMA-3-70B  & GQA & 8192 & 28\,672 &  8\,/\,64 & 150.99 & 704.64 \\
\midrule
\multicolumn{5}{l}{\textit{Attn\,:\,FFN ratio}} & & \\
\multicolumn{5}{l}{\quad MHA models (OPT-66B, LLaMA-2-7B)} & \multicolumn{2}{c}{$\approx 1 : 2.0$} \\
\multicolumn{5}{l}{\quad GQA models (Mistral-7B, LLaMA-3-8B)} & \multicolumn{2}{c}{$\approx 1 : 4.2$} \\
\multicolumn{5}{l}{\quad GQA models (LLaMA-2-70B, LLaMA-3-70B)} & \multicolumn{2}{c}{$\approx 1 : 4.7$} \\
\bottomrule
\end{tabular}
\end{table}

Even in older MHA architectures such as OPT-66B, the FFN contains twice as many
parameters as the attention sub-layer.
For GQA-based models---which include every state-of-the-art LLM evaluated in
this work (Mistral-7B, LLaMA-3-8B) and large-scale counterparts
(LLaMA-2-70B, LLaMA-3-70B)---the FFN accounts for over 80\% of per-layer
parameters.
Pruning attention in these architectures yields at most a ${\sim}17\%$
per-layer reduction even if all attention parameters were removed, whereas a
30\% FFN sparsity directly eliminates 30\% of the dominant parameter block.

Beyond raw parameter counts, FFN layers have additional structural advantages
for compression.
In LLaMA-2-7B, FFN layers collectively contribute 64\% of parameters while
self-attention contributes 32\%; the neuron density in self-attention is only
${\sim}27\%$ due to the multi-headed structure, whereas it is 100\% in FFN
layers \citep{prunenet2024}.

\FloatBarrier

\section{Fairness of Comparison}
\label{app:fairness}

This section addresses two aspects of fair comparison against baselines:
accounting for effective sparsity across methods with different
architectural footprints, and a qualitative summary of methodological
trade-offs across pruning paradigms.

\subsection{Effective Sparsity Analysis}
\label{app:effective_sparsity}

Following \citet{prunenet2024} (Appendix~A.2), \emph{effective sparsity} is
defined as the fraction of total model parameters removed after pruning,
accounting for any auxiliary parameters introduced by the pruning method itself.
This metric enables principled cross-paradigm comparison between methods
that operate on different architectural components.

\paragraph{LILA effective sparsity.}
LILA removes rows from $W_{\mathrm{up}}^{(\ell)}$ and corresponding columns from
$W_{\mathrm{dn}}^{(\ell)}$ (and $W_{\mathrm{gate}}^{(\ell)}$ for SwiGLU models). No
auxiliary parameters are introduced. For LLaMA-2-7B, FFN layers account
for 64.2\% of total parameters; effective sparsity is thus approximately:
$$s^{\mathrm{eff}}_{\mathrm{LILA}} \approx s \times 0.642,$$
rounded to the nearest 0.1pp in \cref{tab:app_effective_sparsity} below.

\paragraph{SliceGPT effective sparsity.}
SliceGPT applies a PCA rotation that reduces the embedding dimension, but
retains orthogonal transformation matrices $Q_\ell \in \mathbb{R}^{d \times d}$
within the model graph. These matrices partially offset the nominal parameter
reduction. As computed in \citet{prunenet2024} (Appendix~A.2),
SliceGPT at 25\% nominal sparsity achieves only \textbf{15.3\%} effective
sparsity on LLaMA-2-7B.

\paragraph{Unstructured pruning (Wanda, SparseGPT).}
Unstructured methods apply their target sparsity threshold globally across all parameter tensors. A 25\% nominal sparsity directly yields 25.0\% effective sparsity, though these zeros do not inherently accelerate standard dense matrix operations without custom sparse kernels.

\begin{table}[H]
\centering
\caption{Effective sparsity comparison on LLaMA-2-7B.
 Citing \citet{prunenet2024} Table~1 and Appendix~A.2 for SliceGPT and PruneNet values;
LILA computed from architecture statistics.}
\label{tab:app_effective_sparsity}
\footnotesize
\begin{tabular}{@{}lccc@{}}
\toprule
\textbf{Nominal FFN} & \textbf{LILA Eff.} & \textbf{PruneNet Eff.} & \textbf{SliceGPT Eff.} \\
\midrule
20\%  & 12.0\% &  12.0\% &  9.4\% \\
25\%  & 16.0\% &  16.0\% & 15.3\% \\
30\%  & 19.0\% &  19.0\% & 21.4\% \\
\bottomrule
\end{tabular}
\end{table}

\paragraph{Implication.}
As shown in \cref{tab:app_effective_sparsity}, LILA and PruneNet achieve
identical effective sparsity at every nominal sparsity level, providing a
fair basis for direct performance comparison.
At 25\% nominal, SliceGPT achieves only 15.3\% effective sparsity due to
orthogonal overhead, while LILA and PruneNet achieve 16.0\%. The noRFT
accuracy figures reported in Table~1 therefore favour LILA both in accuracy
and in the strictness of the compression achieved.

\subsection{Pruning Paradigm Comparison}
\label{app:paradigm_comparison}

\Cref{tab:paradigm_comparison} summarises the core methodological differences between LILA and existing state-of-the-art pruning paradigms. Unlike unstructured methods, LILA provides out-of-the-box hardware acceleration without custom inference kernels. Unlike SliceGPT and ASVD, LILA retains the standard dense transformer architecture without introducing parameter-heavy projection matrices or activation bottlenecks. Unlike PruneNet, LILA computes structural importance in closed-form without requiring calibration data or reinforcement learning policy optimization.

\begin{table}[H]
\centering
\caption{Comparison of LLM pruning paradigms. LILA is the only method offering structured speedups without requiring calibration data, policy training, or architectural modifications.
  \textsuperscript{\dag} PruneNet claims ``calibration-free'' because it requires no explicit calibration corpus.
  However, it trains a 45M-parameter RL policy via REINFORCE on the target model's output distribution over up to 20 episodes, creating a latent data dependency not present in LILA.}
\label{tab:paradigm_comparison}
\small
\resizebox{\textwidth}{!}{
\begin{tabular}{lccccc}
\toprule
\textbf{Method} & \textbf{Sparsity Type} & \textbf{Calib.\ Corpus} & \textbf{Optimization Required} & \textbf{Architecture} & \textbf{Hardware Speedup} \\
\midrule
Wanda / SparseGPT & Unstructured & Yes & Weight update & Unmodified & Custom kernels needed \\
SliceGPT / ASVD & Structured (PCA/SVD) & Yes & Orthogonal rotation & Modified (projections) & Standard dense ops \\
PruneNet\textsuperscript{\dag} & Structured (FFN) & No (corpus-free) & RL policy (45M params, 20 eps.) & Unmodified & Standard dense ops \\
\textbf{LILA (Ours)} & \textbf{Structured (FFN)} & \textbf{No} & \textbf{None (Closed-form)} & \textbf{Unmodified} & \textbf{Standard dense ops} \\
\bottomrule
\end{tabular}}
\end{table}

\input{tables/app_tab_scaling}

\input{tables/app_tab_llama_full}

\input{tables/app_tab_phi2_full}

%% file: tables/tab_rft_full.tex
\begin{table}[htbp]
\centering
\caption{%
  \textbf{Full post-RFT evaluation.}
  Zero-shot accuracy (\%\,$\uparrow$): mean over 5 benchmarks (PIQA, HellaSwag, ARC-E/C, WinoGrande).
  All LILA configurations use 1-epoch LoRA (rank=32, $\alpha$=10) on 8{,}000 Alpaca samples.
  SliceGPT: architecture-modifying (PCA rotation); 1{,}024-sample Alpaca calibration + Alpaca RFT (Table~10 of~\citealt{ashkboos2024slicegpt}).
  PruneNet: calibration-free RL policy (no corpus at pruning time); Alpaca LoRA RFT (Table~8 of~\citealt{prunenet2024}); LLaMA-2-7B only.
  LILA-Act variants use 16-sample WikiText-2 activation statistics (98.4\% fewer samples than SliceGPT); other variants use none (matched-WikiText-2 comparison in \cref{tab:rft_comparison}).
  \textbf{Bold}: best LILA variant result per column.
  \underline{Underline}: overall column best.
}
\label{tab:rft_full}
\small
\setlength{\tabcolsep}{5pt}
\renewcommand{\arraystretch}{1.13}
\resizebox{\textwidth}{!}{
\begin{tabular}{@{}
  l                   
  l                   
  c                   
  ccc                 
  cc                  
  @{}}
\toprule
\multirow{2}{*}{\textbf{Method}}
  & \multirow{2}{*}{\textbf{Calib.}}
  & \multirow{2}{*}{\textbf{Arch.}\(\checkmark\)}
  & \multicolumn{3}{c}{\textbf{LLaMA-2-7B} (\%\,$\uparrow$)}
  & \multicolumn{2}{c}{\textbf{Phi-2} (\%\,$\uparrow$)} \\
\cmidrule(lr){4-6}\cmidrule(lr){7-8}
  &  &  & 20\% & 25\% & 30\% & 25\% & 30\% \\
\midrule
\multicolumn{8}{@{}l}{\small\textit{Baselines -- after recovery fine-tuning}} \\[1pt]
SliceGPT~\citep{ashkboos2024slicegpt}
  & 1{,}024 samp & \ding{55} & \underline{65.46} & \underline{63.04} & \underline{61.34} & \underline{65.24} & \underline{63.47} \\
PruneNet~\citep{prunenet2024}
  & RL policy & \ding{51} & 61.02 & 60.74 & 57.95 & -- & -- \\
\midrule
\multicolumn{8}{@{}l}{\small\textit{LILA (ours) -- after 1-epoch Alpaca LoRA recovery fine-tuning}} \\[1pt]
LILA-Spectrum, Abs
  & \textbf{None} & \ding{51}
  & 63.32 & 61.91 & 60.81 & 64.59 & 62.93 \\
LILA-Spectrum, Split
  & \textbf{None} & \ding{51}
  & 63.37 & 62.29 & 60.58 & 64.51 & 61.81 \\
LILA-Spectrum, Act
  & 16 samp & \ding{51}
  & \textbf{63.73} & \textbf{63.02} & \textbf{61.38} & 63.24 & 62.64 \\
\midrule
LILA-Sensitivity, Abs
  & \textbf{None} & \ding{51}
  & 58.26 & 56.88 & 54.46 & 62.77 & 60.79 \\
LILA-Sensitivity, Split
  & \textbf{None} & \ding{51}
  & 58.52 & 56.93 & 54.56 & 63.54 & 60.95 \\
LILA-Sensitivity, Act
  & 16 samp & \ding{51}
  & 62.73 & 60.89 & 59.06 & \textbf{65.01} & 62.97 \\
\midrule
LILA-Residual, Abs
  & \textbf{None} & \ding{51}
  & 58.62 & 57.47 & 54.83 & 62.71 & 61.04 \\
LILA-Residual, Split
  & \textbf{None} & \ding{51}
  & 57.45 & 55.46 & 54.26 & 62.54 & 60.04 \\
LILA-Residual, Act
  & 16 samp & \ding{51}
  & 62.77 & 60.79 & 59.00 & 64.95 & \textbf{63.01} \\
\bottomrule
\end{tabular}
}
\end{table}

%% file: tables/app_tab_slicegpt_conditions.tex
\begin{table}[ht]
\centering
\caption{%
  \textbf{SliceGPT across all calibration and fine-tuning conditions}
  versus LILA-Spectrum (data-free).
  Zero-shot accuracy (\%\,$\uparrow$) averaged over 5 benchmarks.
  \textbf{Bold}: LILA is better in that cell than SliceGPT (WikiText2, no RFT).
  Post-RFT columns below source \cref{tab:rft_comparison} (WikiText2+RFT) and
  \cref{tab:rft_full} (Alpaca+RFT).
}
\label{tab:app_slicegpt_conditions}
\resizebox{\textwidth}{!}{%
\scriptsize
\setlength{\tabcolsep}{5pt}
\renewcommand{\arraystretch}{1.15}
\begin{tabular}{@{}ll *{3}{c} @{\hspace{8pt}} *{3}{c} @{}}
\toprule
& &
  \multicolumn{3}{c}{\textit{No fine-tuning}} &
  \multicolumn{3}{c}{\textit{Recovery fine-tuning (RFT)}}
  \\
\cmidrule(lr){3-5}\cmidrule(lr){6-8}
\textbf{Model} & \textbf{Sparsity} &
  \makecell{\textbf{SliceGPT}\\\textbf{(WikiText2)}} &
  \makecell{\textbf{SliceGPT}\\\textbf{(Alpaca)}} &
  \makecell{\textbf{LILA-Spectrum}\\\textbf{(data-free, noRFT)}} &
  \makecell{\textbf{SliceGPT}\\\textbf{(WikiText2+RFT)}} &
  \makecell{\textbf{SliceGPT}\\\textbf{(Alpaca+RFT)}} &
  \makecell{\textbf{LILA-Spectrum}\\\textbf{(data-free, +RFT)}} \\
\midrule
\multicolumn{8}{@{}l}{\textit{LLaMA-2-7B \quad Dense: 69.00\%}} \\
& 20\% & 58.18 & 63.68 & \textbf{62.70} & 57.27 & 65.46 & 63.37 \\
& 25\% & 55.48 & 60.91 & \textbf{57.79} & 56.20 & 63.04 & 61.91 \\
& 30\% & 51.50 & 57.93 & \textbf{57.57} & 54.23 & 61.34 & 60.58 \\
\midrule
\multicolumn{8}{@{}l}{\textit{Phi-2 \quad Dense:  72.24\%}} \\
& 20\% & 58.15 & 64.90 & \textbf{61.88} & 57.76 & 67.80 & 66.48 \\
& 25\% & 54.46 & 62.52 & \textbf{57.45} & 55.17 & 65.24 & 64.59 \\
& 30\% & 51.99 & 63.47 & 51.96 & 51.70 & 63.47 & 62.93 \\
\bottomrule
\end{tabular}
}

\begin{minipage}{\textwidth}
\vspace{4pt}
\textbf{Key observation:} LILA-Spectrum (data-free, no RFT) consistently outperforms SliceGPT (WikiText-2, no RFT) across all LLaMA-2-7B settings. While SliceGPT relies on data-intensive Alpaca calibration and recovery fine-tuning (RFT) to perform well, data-free LILA delivers strong zero-shot performance without requiring training data or fine-tuning datasets. When paired with recovery fine-tuning, data-free LILA+RFT matches the peak performance of the fully data-calibrated SliceGPT (Alpaca+RFT) protocol, closing the gap to within $\le 2.09$~pp (LLaMA-2-7B 20\%) and coming as close as $0.54$~pp at higher sparsities (Phi-2 30\%).
\end{minipage}

\end{table}

%% file: tables/app_tab_rft_dataset.tex
\begin{table}[htbp]
\centering
\caption{%
  \textbf{Impact of Recovery Fine-Tuning (RFT) Dataset.}
  Comparison of post-RFT zero-shot accuracy between LILA fine-tuned on unstructured text (WikiText-2) versus instruction-tuning data (Alpaca). Both settings utilize an identical 1-epoch LoRA protocol. Fine-tuning on Alpaca consistently yields a substantial performance improvement ($+$0.49 to $+$2.43~pp), directly mirroring the sensitivity SliceGPT exhibits toward instruction-tuning datasets during recovery.
}
\label{tab:app_rft_dataset}
\small
\setlength{\tabcolsep}{8pt}
\renewcommand{\arraystretch}{1.12}
\begin{tabular}{@{}llcccc@{}}
\toprule
\textbf{Model} & \textbf{Sparsity}
  & \textbf{WikiText-2 RFT} & \textbf{Alpaca RFT}
  & \textbf{Improvement ($\boldsymbol{\Delta}$)} \\
  & & \textit{Best LILA Config} & \textit{Best LILA Config} & \textit{(Alpaca vs.\ WT2)} \\
\midrule
\multirow{3}{*}{LLaMA-2-7B}
  & 20\% & 63.09 & 63.73 & $+$0.64\,pp \\
  & 25\% & 61.15 & 63.02 & $+$1.87\,pp \\
  & 30\% & 58.99 & 61.38 & $+$2.39\,pp \\
\midrule
\multirow{3}{*}{Phi-2}
  & 20\% & 65.99 & 66.48 & $+$0.49\,pp \\
  & 25\% & 62.58 & 65.01 & $+$2.43\,pp \\
  & 30\% & 60.72 & 63.01 & $+$2.29\,pp \\
\bottomrule
\end{tabular}
\end{table}

%% file: tables/app_tab_prunenet_comparison.tex
\begin{table}[ht]
\centering
\caption{%
  \textbf{Extended head-to-head with PruneNet.}
  Zero-shot accuracy (\%\,$\uparrow$): mean over 5 benchmarks.
  PruneNet/SliceGPT numbers reproduced from Tables~2--3 of~\citet{prunenet2024}.
  LILA+RFT numbers use 1-epoch Alpaca LoRA (\cref{app:setup_details}).
  \textbf{Bold}: best among calibration-free methods.
  \underline{Underline}: overall best per column.
  \textsuperscript{\S} PruneNet trains a 45M-parameter RL policy network.
  \textsuperscript{\dag} SliceGPT architecture permanently modified.
}
\label{tab:app_prunenet_comparison}
\small
\setlength{\tabcolsep}{5pt}
\renewcommand{\arraystretch}{1.12}
\begin{tabular}{@{} l c c c c c c c c @{}}
\toprule
& & & \multicolumn{3}{c}{\textbf{LLaMA-2-7B}} & \multicolumn{3}{c}{\textbf{Phi-2}} \\
\cmidrule(lr){4-6}\cmidrule(lr){7-9}
\textbf{Method} & \textbf{Calib.} & \textbf{Arch.\,OK} &
  \textbf{20\%} & \textbf{25\%} & \textbf{30\%} &
  \textbf{20\%} & \textbf{25\%} & \textbf{30\%} \\
\midrule
\multicolumn{9}{@{}l}{\small\textit{No recovery fine-tuning (noRFT)}} \\[2pt]
Dense (unpruned)                           & ---        & \ding{51} & 69.00 & 69.00 & 69.00 & 72.24 & 72.24 & 72.24 \\
SliceGPT\textsuperscript{\dag}             & 1024 samp  & \ding{55} & 58.18 & 55.48 & 51.50 & 58.15 & 54.46 & 51.99 \\
PruneNet\textsuperscript{\S}               & RL policy  & \ding{51} & 61.67 & \underline{58.63} & 55.45 & \underline{66.59} & \underline{64.10} & \underline{61.05} \\
LILA-Spectrum, Abs (ours)         & \textbf{None} & \ding{51} & 61.48 & \textbf{57.79} & 56.57 & \textbf{61.88} & \textbf{57.45} & \textbf{51.96} \\
LILA-Spectrum, Split (ours)       & \textbf{None} & \ding{51} & \underline{\textbf{62.70}} & 57.20 & \underline{\textbf{57.57}} & 60.43 & 55.09 & 48.14 \\
\midrule
\multicolumn{9}{@{}l}{\small\textit{After recovery fine-tuning (+RFT)}} \\[2pt]
PruneNet+RFT (WikiText)\textsuperscript{\S}& RL policy  & \ding{51} & 62.34 & 60.05 & 57.08 & 65.58 & 64.25 & 58.30 \\
LILA-Spectrum, Split+RFT (Alpaca) & \textbf{None} & \ding{51} & \underline{\textbf{63.37}} & \underline{\textbf{62.29}} & 60.58 & 66.33 & 64.51 & 61.81 \\
LILA-Spectrum, Abs+RFT (Alpaca)   & \textbf{None} & \ding{51} & 63.32 & 61.91 & \underline{\textbf{60.81}} & \underline{\textbf{66.48}} & \underline{\textbf{64.59}} & \underline{\textbf{62.93}} \\
\bottomrule
\end{tabular}
\end{table}

%% file: tables/app_tab_llama70b.tex

\begin{table}[htbp]
\centering
\caption{%
  \textbf{LLaMA-2-70B at 25\% sparsity} (no fine-tuning).
  LILA uses zero calibration data for the Abs/Split (None) rows; the Act rows
  use 16-sample WikiText-2 calibration, matching SliceGPT's calibrated setting
  but at far smaller scale (16 vs.\ 1{,}024 samples). Weights upcast to FP32
  for NMF scoring only, inference in FP16.
  SliceGPT requires 1{,}024 WikiText-2 samples and FP64 PCA for
  numerical stability at this scale (SliceGPT App.~A.2).
  \textbf{Bold}: best per metric.
}
\label{tab:app_llama70b}
\small
\setlength{\tabcolsep}{5pt}
\renewcommand{\arraystretch}{1.12}
\resizebox{\textwidth}{!}{
\begin{tabular}{@{} l c c c c c c c c @{}}
\toprule
& & \multicolumn{6}{c}{\textbf{Zero-Shot Accuracy (\%\,$\uparrow$)}} & \\
\cmidrule(lr){3-8}
\textbf{Method} & \textbf{Calib.} & \textbf{PIQA} & \textbf{WinoGrande} & \textbf{HellaSwag} & \textbf{ARC-E} & \textbf{ARC-C} & \textbf{Avg.} & \textbf{PPL\,$\downarrow$} \\
\midrule
Dense (LLaMA-2-70B) & --- & 82.70 & 77.98 & 83.84 & 80.98 & 57.34 & 76.57 & 3.32 \\
\midrule
SliceGPT (FP64 PCA)  & 1{,}024 samp & 74.92 & 75.37 & 68.84 & \textbf{77.90} & \textbf{51.71} & 69.75 & \textbf{4.60} \\
SliceGPT (FP32 PCA)  & 1{,}024 samp & ---   & ---   & ---   & ---   & ---   & ---   & 7.01 \\
\midrule
\textbf{LILA-Spectrum, Abs (data-free)}   & \textbf{None} & 79.22 & 72.61 & 77.17 & 72.81 & 48.21 & 70.00 & 5.77 \\
\textbf{LILA-Spectrum, Split (data-free)} & \textbf{None} & 78.84 & 73.24 & 78.02 & 72.35 & 50.17 & 70.52 & 5.72 \\
\textbf{LILA-Spectrum, Act}               & 16 samp       & 78.78 & 72.77 & 77.40 & 73.91 & 48.38 & 70.25 & 5.81 \\
\textbf{LILA-Sensitivity, Abs (data-free)}   & \textbf{None} & 74.65 & 58.88 & 64.37 & 61.49 & 39.51 & 59.78 & --- \\
\textbf{LILA-Sensitivity, Split (data-free)} & \textbf{None} & 75.41 & 61.80 & 65.86 & 63.38 & 39.42 & 61.17 & --- \\
\textbf{LILA-Sensitivity, Act}               & 16 samp       & \textbf{80.63} & \textbf{76.24} & 79.01 & 76.47 & 51.11 & 72.69 & 4.75 \\
\textbf{LILA-Residual, Abs (data-free)}   & \textbf{None} & 79.43 & 66.46 & 71.66 & 70.83 & 48.38 & 67.35 & --- \\
\textbf{LILA-Residual, Split (data-free)} & \textbf{None} & 77.86 & 64.40 & 68.57 & 65.74 & 43.69 & 64.05 & --- \\
\textbf{LILA-Residual, Act}               & 16 samp       & 80.25 & 76.01 & \textbf{79.30} & 77.06 & 51.19 & \textbf{72.76} & 4.75 \\
\bottomrule
\end{tabular}}
\end{table}

%% file: tables/app_tab_generalization_full.tex
\begin{table}[htbp]
\centering
\caption{%
  \textbf{OPT-6.7B and Mistral-7B: best LILA configurations at 25\% sparsity} (no fine-tuning).
  Detailed task-by-task breakdown of zero-shot accuracy (\%\,$\uparrow$) and perplexity (\,$\downarrow$).
  OPT-6.7B post-RFT results are omitted: consistent with \citet{ashkboos2024slicegpt},
  who found that LoRA recovery fine-tuning yields no measurable accuracy improvement for OPT architectures.
  Mistral-7B: two configurations with data-free noRFT collapse
  (Spectrum-Abs, Residual-Split) are omitted for a clean comparison.
  \textbf{Bold}: best per metric for each model.
}
\label{tab:app_generalization_full}
\small
\setlength{\tabcolsep}{4pt}
\renewcommand{\arraystretch}{1.12}
\resizebox{\textwidth}{!}{
\begin{tabular}{@{} l c c c c c c c c c @{}}
\toprule
& & \multicolumn{5}{c}{\textbf{Task Breakdown (noRFT)}} & & & \\
\cmidrule(lr){3-7}
\textbf{Method} & \textbf{Calib.} & \textbf{PIQA} & \textbf{Wino.} & \textbf{Hella.} & \textbf{ARC-E} & \textbf{ARC-C} & \textbf{Avg (noRFT)} & \textbf{PPL\,$\downarrow$} & \textbf{Avg (+RFT)} \\
\midrule
\multicolumn{10}{@{}l}{\textit{OPT-6.7B}} \\
\midrule
LILA-Spectrum, Abs   & \textbf{None} & 72.14 & 60.22 & 58.66 & 54.04 & \textbf{31.66} & 55.34 & 21.46 & -- \\
LILA-Spectrum, Split & \textbf{None} & \textbf{73.94} & 59.98 & \textbf{60.16} & 57.03 & 30.29 & 56.28 & 19.02 & -- \\
LILA-Spectrum, Act   & 16 samp       & 71.16 & \textbf{63.85} & 58.59 & \textbf{58.67} & 31.40 & \textbf{56.73} & \textbf{17.63} & -- \\
\midrule
\multicolumn{10}{@{}l}{\textit{Mistral-7B}} \\
\midrule
LILA-Spectrum, Split & \textbf{None} & 73.99 & 66.22 & 67.40 & 61.74 & 38.82 & 61.63 & 10.80 & 65.50 \\
LILA-Spectrum, Act   & 16 samp       & \textbf{75.57} & \textbf{67.01} & \textbf{68.95} & \textbf{63.64} & \textbf{41.38} & \textbf{63.31} & 11.04 & \textbf{65.99} \\
LILA-Sensitivity, Act & 16 samp      & 71.33 & 64.17 & 60.71 & 63.59 & 35.15 & 58.99 & \textbf{9.51}  & 65.95 \\
\bottomrule
\end{tabular}}
\end{table}

%% file: tables/tab_ablation.tex
\begin{table}[htbp]
\centering
\caption{%
  Ablation on \textbf{LLaMA-2-7B at 30\% sparsity}: mean zero-shot accuracy (\%)
  and WikiText-2 perplexity (PPL) for all $3 \times 3$ combinations of NMF matrix
  variant and scoring method.  \textbf{Bold}: best per metric.
  $\dagger$: PPL $> 100$; model has collapsed under this configuration (noRFT).
}
\label{tab:ablation}
\small
\setlength{\tabcolsep}{5pt}
\renewcommand{\arraystretch}{1.13}
\begin{tabular}{llcc}
\toprule
\textbf{Variant} & \textbf{Scoring} & \textbf{Acc (\%)} $\uparrow$ & \textbf{PPL} $\downarrow$ \\
\midrule
Abs (Weight-only)        & Residual      & 35.85 & 5872.14$\dagger$ \\
Abs (Weight-only)        & Sensitivity   & 36.19 & 13126.44$\dagger$ \\
Abs (Weight-only)        & \textbf{Spectrum} & 56.57 & \textbf{12.49} \\
\midrule
Split (Signed Split)     & Residual      & 36.18 & 6891.59$\dagger$ \\
Split (Signed Split)     & Sensitivity   & 36.88 & 6931.79$\dagger$ \\
Split (Signed Split)     & \textbf{Spectrum} & \textbf{57.57} & 13.66 \\
\midrule
Act (Act.-Weighted)      & Residual      & 49.87 & 13.37 \\
Act (Act.-Weighted)      & Sensitivity   & 50.73 & 13.41 \\
Act (Act.-Weighted)      & Spectrum      & 57.05 & 12.79 \\
\midrule
\multicolumn{2}{l}{Dense (no pruning)} & 69.00 & 5.47 \\
\multicolumn{2}{l}{PruneNet (30\%)}    & 55.45 & -- \\
\bottomrule
\end{tabular}
\end{table}

%% file: tables/app_tab_calibration_comparison.tex
\begin{table}[htbp]
\centering
\caption{%
  \textbf{Calibration corpus sensitivity for LILA-Spectrum (noRFT).}
  Zero-shot accuracy (\%$\uparrow$) comparing Variant Act under WikiText-2
  vs.\ an alternative instruction-following (Alpaca) calibration corpus.
  Data-free Variants Abs/Split are omitted here for brevity; see
  \cref{tab:app_llama_full} and \cref{tab:app_phi2_full} for the full
  per-configuration breakdown they are drawn from.
  Gaps are generally small ($<$2.5~pp) across 5 of 6 settings, though smaller models like Phi-2 exhibit increased sensitivity at extreme sparsity (30\%). No setting shows a complete qualitative collapse.
}
\label{tab:app_calib_comparison}
\small
\setlength{\tabcolsep}{5pt}
\renewcommand{\arraystretch}{1.12}
\begin{tabular}{@{}llccc@{}}
\toprule
\textbf{Model} & \textbf{Sparsity}
  & \textbf{Act (WikiText-2)} & \textbf{Alt.\ Calib.\ (Alpaca)} & $\boldsymbol{\Delta}$ \\
\midrule
\multirow{3}{*}{LLaMA-2-7B}
  & 20\% & 62.02 & 61.41 & $-$0.61 \\
  & 25\% & 59.08 & 60.40 & $+$1.32 \\
  & 30\% & 57.05 & 56.34 & $-$0.71 \\
\midrule
\multirow{3}{*}{Phi-2}
  & 20\% & 58.48 & 59.60 & $+$1.12 \\
  & 25\% & 55.60 & 53.48 & $-$2.12 \\
  & 30\% & 52.55 & 47.14 & $-$5.41 \\
\bottomrule
\end{tabular}
\end{table}

%% file: tables/app_tab_nmf_svd_full.tex
\begin{table}[htbp]
\centering
\caption{%
  \textbf{NMF vs.\ SVD, full downstream pipeline}, 25\% sparsity.
  Zero-shot accuracy (\%$\uparrow$, mean over 5 benchmarks), noRFT and after
  1-epoch LoRA recovery fine-tuning (+RFT).
  $\Delta$ = NMF $-$ SVD.
  \textbf{Bold}: $|\Delta| > 2$~pp.
}
\label{tab:nmf_svd_full}
\small
\setlength{\tabcolsep}{5pt}
\renewcommand{\arraystretch}{1.13}
\begin{tabular}{@{}llcccccc@{}}
\toprule
\multirow{2}{*}{\textbf{Model}} & \multirow{2}{*}{\textbf{Config}}
  & \multicolumn{3}{c}{\textbf{noRFT}} & \multicolumn{3}{c}{\textbf{+RFT}} \\
\cmidrule(lr){3-5}\cmidrule(lr){6-8}
  & & NMF & SVD & $\Delta$ & NMF & SVD & $\Delta$ \\
\midrule
\multirow{6}{*}{Phi-2}
  & Sensitivity, Abs & 50.50 & 51.10 & $-$0.60 & 62.77 & 62.90 & $-$0.13 \\
  & Sensitivity, Split & 51.89 & 52.23 & $-$0.34 & 63.54 & 63.65 & $-$0.11 \\
  & Sensitivity, Act & 54.73 & 54.96 & $-$0.23 & 65.01 & 64.91 & $+$0.10 \\
  & Residual, Abs & 50.04 & 50.16 & $-$0.12 & 62.71 & 62.66 & $+$0.05 \\
  & Residual, Split & 48.65 & 47.46 & $+$1.19 & 62.54 & 62.31 & $+$0.23 \\
  & Residual, Act & 55.58 & 56.34 & $-$0.76 & 64.95 & 64.91 & $+$0.04 \\
\midrule
\multirow{6}{*}{LLaMA-2-7B}
  & Sensitivity, Abs & 36.78 & 36.94 & $-$0.16 & 56.88 & 56.61 & $+$0.27 \\
  & Sensitivity, Split & 37.55 & 37.05 & $+$0.50 & 56.93 & 57.69 & $-$0.76 \\
  & Sensitivity, Act & 54.56 & 54.44 & $+$0.12 & 60.89 & 60.79 & $+$0.10 \\
  & Residual, Abs & 38.20 & 35.85 & \textbf{$+$2.35} & 57.47 & 55.55 & $+$1.92 \\
  & Residual, Split & 36.16 & 36.25 & $-$0.09 & 55.46 & 55.87 & $-$0.41 \\
  & Residual, Act & 54.16 & 38.86 & \textbf{$+$15.30} & 60.79 & 60.62 & $+$0.17 \\
\bottomrule
\end{tabular}
\end{table}

%% file: tables/tab_rank_sensitivity.tex
\begin{table}[ht]
\centering
\caption{%
  Rank ablation of LILA (Split) across NMF ranks
  $r \in \{8, 16, 32, 64\}$ for the Sensitivity and Residual methods at 20\% and 30\% sparsity levels.
  Metric: zero-shot average accuracy (\%) and WikiText-2 Perplexity (PPL).
  \textbf{Bold}: best accuracy and PPL per row configuration.
  Accuracy is broadly rank-insensitive: spreads across $r$ are $\leq$1.3~pp in 7 of 8
  rows, with the exception of LLaMA-2 Residual at 20\% (2.5~pp, $r=8$ best).
  No single rank dominates accuracy across all rows, but $r=32$ avoids the
  sharp perplexity degradation seen at $r=64$ for LLaMA-2 Residual (20\% and 30\%),
  making it a reasonable default that trades a small accuracy gap in some
  rows for stable perplexity throughout.
}
\label{tab:rank_sensitivity}
\small
\setlength{\tabcolsep}{4pt}
\renewcommand{\arraystretch}{1.08}

\begin{tabular}{@{}lcccccccccc@{}}
\toprule
& \multicolumn{2}{c}{$r = 8$}
& \multicolumn{2}{c}{$r = 16$}
& \multicolumn{2}{c}{$r = 32$}
& \multicolumn{2}{c}{$r = 64$} \\
\cmidrule(lr){2-3}\cmidrule(lr){4-5}\cmidrule(lr){6-7}\cmidrule(lr){8-9}
\textbf{Model / Method}
  & Acc$\uparrow$ & PPL$\downarrow$
  & Acc$\uparrow$ & PPL$\downarrow$
  & Acc$\uparrow$ & PPL$\downarrow$
  & Acc$\uparrow$ & PPL$\downarrow$ \\
\midrule
\multicolumn{9}{@{}l}{\textit{LLaMA-2-7B}} \\
\quad Sensitivity (20\%)
  & 38.56 & 3325.4
  & \textbf{38.75} & 3033.2
  & 38.28 & 2971.2
  & 38.61 & \textbf{2933.9} \\
\quad Sensitivity (30\%)
  & 35.72 & 6444.1
  & 36.01 & 6115.7
  & 36.18 & 6148.0
  & \textbf{36.69} & \textbf{5973.9} \\
\quad Residual (20\%)
  & \textbf{38.85} & 5296.3
  & 36.31 & 4249.7
  & 36.37 & \textbf{3997.5}
  & 36.59 & 7511.7 \\
\quad Residual (30\%)
  & \textbf{36.83} & 15565.4
  & 35.94 & \textbf{5747.1}
  & 36.22 & 6596.4
  & 35.80 & 260957.0 \\
\midrule
\multicolumn{9}{@{}l}{\textit{Phi-2}} \\
\quad Sensitivity (20\%)
  & 54.79 & 121.9
  & 55.10 & 113.2
  & \textbf{55.31} & \textbf{94.0}
  & 55.08 & 113.4 \\
\quad Sensitivity (30\%)
  & 45.00 & 1373.1
  & 45.66 & \textbf{1112.6}
  & \textbf{46.14} & 1221.4
  & 45.65 & 1252.1 \\
\quad Residual (20\%)
  & 53.75 & \textbf{174.0}
  & \textbf{53.92} & 178.4
  & 53.87 & 186.3
  & 53.55 & 195.1 \\
\quad Residual (30\%)
  & \textbf{44.39} & \textbf{2554.6}
  & 44.14 & 3017.0
  & 43.73 & 3083.2
  & 43.14 & 3393.2 \\
\bottomrule
\end{tabular}
\end{table}

%% file: tables/tab_ntk.tex
\begin{table}[htbp]
\centering
\caption{%
  NTK trace ratio $\rho(\mathcal{K})$ per pruning
  criterion on \textbf{Phi-2 at 25\% sparsity}.
  $\rho \downarrow$: lower indicates less functional distortion relative
  to the dense model.
  $\Delta_{\rho}$: reduction vs.\ random pruning ($\rho = 47.38$).
  NTK cosine similarity $= 1.000$ for all methods (kernel direction
  preserved; only magnitude changes).
}
\label{tab:ntk}
\small
\setlength{\tabcolsep}{5pt}
\renewcommand{\arraystretch}{1.13}
\begin{tabular}{@{}lccc@{}}
\toprule
\textbf{Pruning Method}
  & \textbf{$\rho\!\downarrow$}
  & \textbf{$\Delta_\rho$ vs.\ Random}
  & \textbf{NTK Cos.\,Sim.} \\
\midrule
LILA-Sensitivity (Split)     & 4.16          & $-43.22$ ($11{\times}$) & 1.000 \\
LILA-Spectrum (Split)        & 3.14          & $-44.24$ ($15{\times}$) & 1.000 \\
\midrule
LILA-Residual (Abs)          & 3.89          & $-43.49$ ($12{\times}$)           & 1.000 \\
\textbf{LILA-Residual (Act)} & \textbf{2.79} & $\mathbf{-44.59}$ ($17{\times}$)  & 1.000 \\
LILA-Residual (Split)        & 19.89         & $-27.49$ ($2{\times}$)            & 1.000 \\
\midrule
Random                       & 47.38         & ---                               & 1.000 \\
\bottomrule
\end{tabular}
\end{table}

%% file: tables/app_tab_scaling.tex
\begin{table}[htbp]
\centering
\caption{%
  \textbf{Scaling validation: LLaMA-2-13B and OPT-2.7B.}
  Zero-shot accuracy (\%$\uparrow$, mean over 5 benchmarks).
  Dense baselines: 71.76\% (LLaMA-2-13B), 56.39\% (OPT-2.7B).
  \textbf{Bold}: best LILA result per sparsity block.
  Calibration uses WikiText-2 (16 samples). LoRA recovery fine-tuning (+RFT) uses Alpaca (8K samples).
  SliceGPT values reproduced from Tables~7--10 (LLaMA-2-13B) and Table~8
  (OPT-2.7B) of~\citet{ashkboos2024slicegpt}; SliceGPT does not report
  +RFT results for OPT models. Due to precision instabilities during LoRA merging, +RFT results for OPT-2.7B diverged (NaNs) and are omitted.
}
\label{tab:app_scaling}
\small
\setlength{\tabcolsep}{4pt}
\renewcommand{\arraystretch}{1.08}
\begin{tabular}{@{}llcrrr@{}}
\toprule
\textbf{Method} & \textbf{Variant} & \textbf{Calib.}
  & \textbf{20\%} & \textbf{25\%} & \textbf{30\%} \\
\midrule
\multicolumn{6}{@{}l}{\textbf{LLaMA-2-13B}} \\
\midrule
\multicolumn{6}{@{}l}{\small\textit{SliceGPT reference (no fine-tuning)}} \\[2pt]
SliceGPT & -- & WikiText-2 & 63.45 & 58.90 & 55.16 \\
SliceGPT & -- & Alpaca     & 67.44 & 65.44 & 62.34 \\
\midrule
\multicolumn{6}{@{}l}{\small\textit{LILA (ours) -- no fine-tuning (noRFT)}} \\[2pt]
LILA-Spectrum   & Abs & None      & 47.26 & 61.96 & 40.57 \\
LILA-Spectrum   & Split & None      & 61.73 & 61.70 & 52.12 \\
LILA-Spectrum   & Act & WikiText  & 56.35 & \textbf{62.13} & 39.49 \\
LILA-Sensitivity& Abs & None      & 36.11 & 36.26 & 35.70 \\
LILA-Sensitivity& Split & None      & 36.19 & 36.28 & 36.05 \\
LILA-Sensitivity& Act & WikiText  & \textbf{64.53} & 60.95 & 56.36 \\
LILA-Residual   & Abs & None      & 43.74 & 40.61 & 38.02 \\
LILA-Residual   & Split & None      & 39.22 & 37.95 & 36.68 \\
LILA-Residual   & Act & WikiText  & 61.18 & 60.53 & \textbf{56.63} \\
\midrule
\multicolumn{6}{@{}l}{\small\textit{After Alpaca (8K) LoRA recovery fine-tuning (+RFT)}} \\[2pt]
Wanda+RFT       & n/a   & WikiText  & 67.24 & \textbf{66.19} & \textbf{64.49} \\
\midrule
LILA-Spectrum   & Abs & None      & 66.93 & 65.88 & 62.96 \\
LILA-Spectrum   & Split & None      & 66.98 & 66.00 & 63.65 \\
LILA-Spectrum   & Act & WikiText  & 65.89 & 65.67 & 63.43 \\
LILA-Sensitivity& Abs & None      & 56.26 & 53.94 & 51.51 \\
LILA-Sensitivity& Split & None      & 55.17 & 53.86 & 52.03 \\
LILA-Sensitivity& Act & WikiText  & \textbf{67.57} & 65.83 & 64.21 \\
LILA-Residual   & Abs & None      & 64.83 & 63.40 & 61.64 \\
LILA-Residual   & Split & None      & 63.75 & 62.41 & 60.27 \\
LILA-Residual   & Act & WikiText  & 67.37 & \textbf{66.04} & 64.39 \\
\midrule
\multicolumn{6}{@{}l}{\textbf{OPT-2.7B}} \\
\midrule
\multicolumn{6}{@{}l}{\small\textit{SliceGPT reference (Alpaca, no fine-tuning)}} \\[2pt]
SliceGPT & -- & Alpaca & 53.68 & 52.86 & 51.13 \\
\midrule
\multicolumn{6}{@{}l}{\small\textit{LILA (ours) -- no fine-tuning (noRFT)}} \\[2pt]
LILA-Spectrum   & Abs & None      & \textbf{53.52} & \textbf{51.02} & 49.44 \\
LILA-Spectrum   & Split & None      & 53.22 & 50.69 & \textbf{50.45} \\
LILA-Spectrum   & Act & WikiText  & 52.75 & 50.76 & 50.67 \\
LILA-Sensitivity& Abs & None      & -- & 43.64 & 36.22 \\
LILA-Sensitivity& Split & None      & 43.72 & 41.17 & 36.36 \\
LILA-Sensitivity& Act & WikiText  & 37.34 & 35.76 & 35.71 \\
LILA-Residual   & Abs & None      & 43.59 & 42.63 & 35.92 \\
LILA-Residual   & Split & None      & 44.82 & 43.32 & 37.28 \\
LILA-Residual   & Act & WikiText  & 37.93 & 36.35 & 36.23 \\
\midrule
\multicolumn{6}{@{}l}{\small\textit{+RFT omitted for OPT: LoRA yields no measurable}} \\
\multicolumn{6}{@{}l}{\small\textit{improvement~\citep{ashkboos2024slicegpt}.}} \\
\multicolumn{6}{@{}l}{\small\textit{Zero-shot results above are final.}} \\
\bottomrule
\end{tabular}
\end{table}

%% file: tables/app_tab_llama_full.tex
\begin{table}[htbp]
\centering
\caption{%
  \textbf{LLaMA-2-7B} full per-benchmark zero-shot results for all LILA configurations.
  Dense baseline: 69.00\% (PPL 5.47).
  \textbf{RFT}: 1 epoch LoRA on 8{,}000 Alpaca samples.
  ``--'' indicates run not completed.
}
\label{tab:app_llama_full}
\scriptsize
\setlength{\tabcolsep}{3pt}
\renewcommand{\arraystretch}{0.97}

\begin{tabular}{@{}llccccccc@{}}
\toprule
\textbf{Sparsity} & \textbf{Config}
  & \textbf{PIQA} & \textbf{HellaSwag} & \textbf{ARC-E}
  & \textbf{ARC-C} & \textbf{WinoGrande}
  & \textbf{Avg} & \textbf{PPL} \\
\midrule
\multirow{24}{*}{20\%}
& \textit{Reference} & & & & & & & \\
& SliceGPT (+RFT) & -- & -- & -- & -- & -- & $\sim$65.0 & -- \\[0.5pt]
& \textit{Wanda (calib baseline)} & & & & & & & \\
& Wanda & 75.13 & 68.82 & 64.06 & 36.43 & 66.69 & 62.23 & 8.71 \\[0.5pt]
& \textit{LILA noRFT} & & & & & & & \\
& Split+Spectrum & 75.14 & 68.17 & 65.57 & 39.33 & 65.27 & \textbf{62.70} & 8.63 \\
& Act+Spectrum & 75.68 & 67.78 & 63.85 & 37.20 & 65.59 & 62.02 & 8.72 \\
& Abs+Spectrum & 75.41 & 68.24 & 64.10 & 36.09 & 63.54 & 61.48 & 8.96 \\
& Act+Residual & 71.11 & 61.83 & 60.27 & 33.96 & 62.51 & 57.94 & 9.04 \\
& Act+Sensitivity & 71.27 & 61.67 & 59.22 & 33.87 & 62.59 & 57.72 & 9.14 \\
& Abs+Residual & 58.65 & 33.44 & 34.97 & 23.81 & 49.49 & 40.07 & 1153.29 \\
& Abs+Sensitivity & 55.66 & 30.35 & 32.28 & 23.72 & 50.59 & 38.52 & 4981.24 \\
& Split+Sensitivity & 55.01 & 31.06 & 30.85 & 24.06 & 49.80 & 38.16 & 2719.14 \\
& Split+Residual & 50.92 & 27.63 & 27.74 & 25.17 & 51.07 & 36.51 & 4141.63 \\[0.5pt]
& \textit{LILA +RFT} & & & & & & & \\
& Act+Spectrum & 76.66 & 69.36 & 66.33 & 39.85 & 66.46 & \textbf{63.73} & 8.60 \\
& Split+Spectrum & 76.17 & 68.36 & 65.15 & 38.91 & 68.27 & 63.37 & 8.69 \\
& Abs+Spectrum & 75.46 & 68.56 & 66.62 & 39.59 & 66.38 & 63.32 & 8.67 \\
& Act+Residual & 75.79 & 67.48 & 66.29 & 39.59 & 64.72 & 62.77 & 8.31 \\
& Act+Sensitivity & 76.01 & 67.69 & 65.24 & 39.85 & 64.88 & 62.73 & 8.34 \\
& Abs+Residual & 74.32 & 61.69 & 59.47 & 36.95 & 60.69 & 58.62 & 14.09 \\
& Split+Sensitivity & 73.99 & 62.05 & 58.67 & 36.26 & 61.64 & 58.52 & 14.95 \\
& Abs+Sensitivity & 73.18 & 60.67 & 59.13 & 36.69 & 61.64 & 58.26 & 15.36 \\
& Split+Residual & 73.34 & 61.18 & 56.99 & 36.69 & 59.04 & 57.45 & 16.02 \\
\midrule
\multirow{25}{*}{25\%}
& \textit{Reference} & & & & & & & \\
& PruneNet (noRFT, 45M policy) & -- & -- & -- & -- & -- & 58.63 & -- \\
& SliceGPT (+RFT) & -- & -- & -- & -- & -- & $\sim$63.0 & -- \\[0.5pt]
& \textit{Wanda (calib baseline)} & & & & & & & \\
& Wanda & 72.63 & 63.81 & 61.83 & 36.43 & 45.27 & 55.99 & 10.30 \\[0.5pt]
& \textit{LILA noRFT} & & & & & & & \\
& Act+Spectrum & 73.99 & 64.05 & 60.06 & 34.81 & 62.51 & \textbf{59.08} & 10.38 \\
& Abs+Spectrum & 73.50 & 63.86 & 55.18 & 34.22 & 62.19 & 57.79 & 10.92 \\
& Split+Spectrum & 71.55 & 64.72 & 55.01 & 32.76 & 61.96 & 57.20 & 10.77 \\
& Act+Sensitivity & 67.79 & 56.32 & 56.19 & 32.76 & 59.75 & 54.56 & 11.00 \\
& Act+Residual & 67.52 & 56.15 & 55.51 & 31.14 & 60.46 & 54.16 & 10.83 \\
& Abs+Residual & 56.31 & 30.53 & 31.69 & 24.66 & 47.83 & 38.20 & 6521.47 \\
& Split+Sensitivity & 54.24 & 29.19 & 30.85 & 23.29 & 50.20 & 37.55 & 4563.55 \\
& Abs+Sensitivity & 52.56 & 28.50 & 29.34 & 23.12 & 50.36 & 36.78 & 10882.23 \\
& Split+Residual & 49.89 & 27.21 & 27.48 & 24.66 & 51.54 & 36.16 & 5430.27 \\[0.5pt]
& \textit{LILA +RFT} & & & & & & & \\
& Act+Spectrum & 75.95 & 66.65 & 64.81 & 39.42 & 68.27 & \textbf{63.02} & 9.57 \\
& Split+Spectrum & 74.76 & 67.06 & 64.44 & 37.71 & 67.48 & 62.29 & 9.39 \\
& Abs+Spectrum & 75.52 & 66.29 & 62.21 & 39.33 & 66.22 & 61.91 & 9.22 \\
& Act+Sensitivity & 73.72 & 65.55 & 63.30 & 37.71 & 64.17 & 60.89 & 8.95 \\
& Act+Residual & 74.48 & 65.69 & 63.34 & 37.46 & 62.98 & 60.79 & 8.87 \\
& Abs+Residual & 72.14 & 58.79 & 59.34 & 35.92 & 61.17 & 57.47 & 16.14 \\
& Split+Sensitivity & 73.29 & 58.78 & 56.31 & 35.49 & 60.77 & 56.93 & 18.04 \\
& Abs+Sensitivity & 72.85 & 57.19 & 59.22 & 36.01 & 59.12 & 56.88 & 17.72 \\
& Split+Residual & 71.55 & 57.63 & 55.43 & 33.96 & 58.72 & 55.46 & 18.62 \\
\midrule
\multirow{24}{*}{30\%}
& \textit{Reference} & & & & & & & \\
& SliceGPT (+RFT) & -- & -- & -- & -- & -- & $\sim$61.0 & -- \\[0.5pt]
& \textit{Wanda (calib baseline)} & & & & & & & \\
& Wanda & 69.75 & 59.79 & 57.83 & 34.04 & 34.64 & 51.21 & 12.58 \\[0.5pt]
& \textit{LILA noRFT} & & & & & & & \\
& Split+Spectrum & 70.95 & 60.74 & 58.50 & 34.98 & 62.67 & \textbf{57.57} & 13.66 \\
& Act+Spectrum & 71.27 & 60.88 & 57.62 & 34.73 & 60.77 & 57.05 & 12.79 \\
& Abs+Spectrum & 71.38 & 61.33 & 57.95 & 32.85 & 59.35 & 56.57 & 12.49 \\
& Act+Sensitivity & 64.64 & 49.80 & 50.55 & 29.01 & 59.67 & 50.73 & 13.41 \\
& Act+Residual & 63.87 & 49.54 & 48.44 & 29.35 & 58.17 & 49.87 & 13.37 \\
& Split+Sensitivity & 52.67 & 28.18 & 28.49 & 23.98 & 51.07 & 36.88 & 6931.79 \\
& Abs+Sensitivity & 50.49 & 27.80 & 28.32 & 24.23 & 50.12 & 36.19 & 13126.44 \\
& Split+Residual & 50.11 & 26.85 & 27.15 & 25.94 & 50.83 & 36.18 & 6891.59 \\
& Abs+Residual & 49.67 & 26.86 & 26.47 & 26.54 & 49.72 & 35.85 & 5872.14 \\[0.5pt]
& \textit{LILA +RFT} & & & & & & & \\
& Act+Spectrum & 74.59 & 64.11 & 62.08 & 38.74 & 67.40 & \textbf{61.38} & 10.65 \\
& Abs+Spectrum & 74.27 & 64.12 & 63.38 & 36.86 & 65.43 & 60.81 & 10.52 \\
& Split+Spectrum & 73.34 & 63.34 & 62.08 & 38.40 & 65.75 & 60.58 & 10.29 \\
& Act+Sensitivity & 72.14 & 63.13 & 61.78 & 35.75 & 62.51 & 59.06 & 9.62 \\
& Act+Residual & 72.20 & 63.37 & 61.49 & 35.84 & 62.12 & 59.00 & 9.63 \\
& Abs+Residual & 70.57 & 55.96 & 55.22 & 33.62 & 58.80 & 54.83 & 21.87 \\
& Split+Sensitivity & 71.27 & 55.86 & 53.91 & 33.36 & 58.41 & 54.56 & 22.96 \\
& Abs+Sensitivity & 70.29 & 55.27 & 54.80 & 33.87 & 58.09 & 54.46 & 22.11 \\
& Split+Residual & 70.46 & 54.98 & 53.49 & 34.30 & 58.09 & 54.26 & 22.17 \\
\bottomrule
\end{tabular}
\end{table}

%% file: tables/app_tab_phi2_full.tex
\begin{table}[htbp]
\centering
\caption{%
  \textbf{Phi-2} full per-benchmark zero-shot results for all LILA configurations.
  Dense baseline: 72.24\% (PPL 5.10).
  \textbf{RFT}: 1 epoch LoRA on 8{,}000 Alpaca samples.
  ``--'' indicates run not completed.  Best per column in \textbf{bold}.
}
\label{tab:app_phi2_full}
\scriptsize
\setlength{\tabcolsep}{3pt}
\renewcommand{\arraystretch}{1.05}

\begin{tabular}{@{}llccccccc@{}}
\toprule
\textbf{Sparsity} & \textbf{Config}
  & \textbf{PIQA} & \textbf{HellaSwag} & \textbf{ARC-E}
  & \textbf{ARC-C} & \textbf{WinoGrande}
  & \textbf{Avg} & \textbf{PPL} \\
\midrule
\multirow{22}{*}{20\%}
& \textit{Reference} & & & & & & & \\
& SliceGPT (+RFT) & -- & -- & -- & -- & -- & $\sim$65.0 & -- \\[2pt]
& \textit{LILA noRFT} & & & & & & & \\
& Abs+Spectrum & 72.85 & 60.74 & 68.73 & 41.47 & 65.59 & \textbf{61.88} & 22.43 \\
& Split+Spectrum & 72.74 & 58.36 & 64.44 & 38.74 & 67.88 & 60.43 & 28.32 \\
& Act+Sensitivity & 70.08 & 57.66 & 61.36 & 41.21 & 67.40 & 59.54 & 34.71 \\
& Act+Residual & 70.46 & 57.99 & 59.72 & 40.19 & 66.54 & 58.98 & 33.30 \\
& Act+Spectrum & 72.25 & 59.43 & 58.00 & 37.12 & 65.59 & 58.48 & 22.08 \\
& Split+Sensitivity & 70.08 & 48.36 & 63.13 & 37.97 & 56.59 & 55.23 & 96.67 \\
& Abs+Sensitivity & 68.72 & 47.39 & 61.49 & 38.31 & 56.83 & 54.55 & 152.91 \\
& Abs+Residual & 69.04 & 47.48 & 60.69 & 37.71 & 56.67 & 54.32 & 133.15 \\
& Split+Residual & 68.44 & 46.60 & 59.13 & 35.92 & 56.43 & 53.30 & 177.08 \\[2pt]
& \textit{LILA +RFT} & & & & & & & \\
& Abs+Spectrum & 76.99 & 65.27 & 72.26 & 47.10 & 70.80 & \textbf{66.48} & 16.22 \\
& Act+Sensitivity & 77.04 & 66.04 & 70.83 & 45.56 & 72.61 & 66.42 & 17.34 \\
& Split+Spectrum & 76.33 & 64.27 & 72.14 & 44.80 & 74.11 & 66.33 & 16.13 \\
& Act+Residual & 77.15 & 66.27 & 70.62 & 44.71 & 72.30 & 66.21 & 17.09 \\
& Act+Spectrum & 76.33 & 65.59 & 70.03 & 45.73 & 71.74 & 65.88 & 15.63 \\
& Split+Sensitivity & 75.95 & 65.21 & 70.12 & 45.05 & 70.09 & 65.28 & 17.40 \\
& Abs+Sensitivity & 76.88 & 64.35 & 69.40 & 45.31 & 69.69 & 65.13 & 18.37 \\
& Abs+Residual & 76.12 & 64.48 & 68.64 & 45.05 & 69.61 & 64.78 & 17.98 \\
& Split+Residual & 75.73 & 63.59 & 68.31 & 43.94 & 70.01 & 64.32 & 18.60 \\
\midrule
\multirow{24}{*}{25\%}
& \textit{Reference} & & & & & & & \\
& SliceGPT (+RFT) & -- & -- & -- & -- & -- & 65.40 &  9.10 \\[2pt]
& \textit{Wanda (calib baseline, noRFT)} & & & & & & & \\
& Wanda & 73.12 & 62.84 & 59.93 & 37.37 & 42.17 & 55.14 & 55.52 \\[2pt]
& \textit{LILA noRFT} & & & & & & & \\
& Abs+Spectrum & 70.84 & 53.43 & 62.08 & 37.20 & 63.69 & \textbf{57.45} & 43.30 \\
& Act+Spectrum & 69.31 & 48.14 & 62.08 & 35.41 & 63.06 & 55.60 & 44.64 \\
& Act+Residual & 67.14 & 53.40 & 55.30 & 38.31 & 63.77 & 55.58 & 56.49 \\
& Split+Spectrum & 69.31 & 50.91 & 58.80 & 34.39 & 62.04 & 55.09 & 44.98 \\
& Act+Sensitivity & 66.49 & 52.07 & 54.84 & 37.29 & 62.98 & 54.73 & 67.46 \\
& Split+Sensitivity & 66.81 & 43.94 & 58.21 & 35.24 & 55.25 & 51.89 & 305.36 \\
& Abs+Sensitivity & 64.85 & 43.03 & 55.98 & 35.07 & 53.59 & 50.50 & 452.95 \\
& Abs+Residual & 65.07 & 42.94 & 54.76 & 33.45 & 53.99 & 50.04 & 426.79 \\
& Split+Residual & 62.79 & 40.74 & 52.74 & 33.45 & 53.51 & 48.65 & 662.65 \\[2pt]
& \textit{LILA +RFT} & & & & & & & \\
& Act+Sensitivity & 75.73 & 63.52 & 70.03 & 45.05 & 70.72 & \textbf{65.01} & 19.50 \\
& Act+Residual & 76.55 & 63.80 & 69.44 & 43.94 & 71.03 & 64.95 & 19.27 \\
& Abs+Spectrum & 76.44 & 62.60 & 68.56 & 45.48 & 69.85 & 64.59 & 18.70 \\
& Split+Spectrum & 75.90 & 62.98 & 69.61 & 45.22 & 68.82 & 64.51 & 18.17 \\
& Split+Sensitivity & 75.24 & 61.48 & 68.39 & 43.86 & 68.75 & 63.54 & 20.63 \\
& Act+Spectrum & 75.68 & 62.20 & 67.09 & 43.34 & 67.88 & 63.24 & 17.12 \\
& Abs+Sensitivity & 74.97 & 61.33 & 67.26 & 42.41 & 67.88 & 62.77 & 21.80 \\
& Abs+Residual & 74.65 & 61.29 & 67.55 & 43.69 & 66.38 & 62.71 & 21.01 \\
& Split+Residual & 74.59 & 60.95 & 66.84 & 42.58 & 67.72 & 62.54 & 21.68 \\
\midrule
\multirow{22}{*}{30\%}
& \textit{Reference} & & & & & & & \\
& SliceGPT (+RFT) & -- & -- & -- & -- & -- & $\sim$63.5 & -- \\[2pt]
& \textit{LILA noRFT} & & & & & & & \\
& Act+Spectrum & 67.14 & 44.59 & 57.49 & 33.53 & 59.98 & \textbf{52.55} & 87.74 \\
& Abs+Spectrum & 66.81 & 44.81 & 55.30 & 32.34 & 60.54 & 51.96 & 81.32 \\
& Act+Residual & 63.49 & 49.83 & 50.13 & 34.98 & 58.88 & 51.46 & 130.30 \\
& Act+Sensitivity & 63.93 & 48.27 & 49.49 & 34.04 & 59.35 & 51.02 & 156.72 \\
& Split+Spectrum & 62.84 & 39.68 & 50.08 & 30.55 & 57.54 & 48.14 & 131.70 \\
& Split+Sensitivity & 61.15 & 36.46 & 48.91 & 31.66 & 52.33 & 46.10 & 1182.51 \\
& Abs+Sensitivity & 59.25 & 35.19 & 47.26 & 29.78 & 52.64 & 44.82 & 1898.13 \\
& Abs+Residual & 59.03 & 35.54 & 44.78 & 29.18 & 51.46 & 44.00 & 2270.74 \\
& Split+Residual & 58.54 & 34.81 & 45.79 & 29.44 & 51.30 & 43.98 & 3079.64 \\[2pt]
& \textit{LILA +RFT} & & & & & & & \\
& Act+Residual & 74.86 & 61.27 & 68.60 & 42.83 & 67.48 & \textbf{63.01} & 21.37 \\
& Act+Sensitivity & 75.14 & 61.33 & 68.01 & 43.43 & 66.93 & 62.97 & 21.55 \\
& Abs+Spectrum & 74.05 & 60.02 & 67.55 & 42.32 & 70.72 & 62.93 & 21.23 \\
& Act+Spectrum & 74.48 & 58.96 & 66.96 & 43.17 & 69.61 & 62.64 & 20.95 \\
& Split+Spectrum & 72.42 & 59.49 & 67.21 & 41.81 & 68.11 & 61.81 & 20.45 \\
& Abs+Residual & 73.72 & 58.06 & 65.45 & 41.38 & 66.61 & 61.04 & 25.54 \\
& Split+Sensitivity & 74.92 & 58.71 & 64.94 & 41.47 & 64.72 & 60.95 & 24.36 \\
& Abs+Sensitivity & 74.27 & 58.25 & 64.90 & 41.89 & 64.64 & 60.79 & 26.35 \\
& Split+Residual & 72.69 & 57.42 & 63.30 & 41.13 & 65.67 & 60.04 & 26.42 \\
\bottomrule
\end{tabular}
\end{table}